\documentclass[11pt]{article}

\usepackage[sort]{cite}

\usepackage{amssymb, amsthm, mathtools, todonotes}
\usepackage{dsfont}
\usepackage{relsize}
\usepackage[shortcuts]{extdash}
\usepackage{enumerate}
\usepackage{thm-restate}
\usepackage{authblk}
\usepackage{fullpage}
\usepackage{hyperref}
\usepackage{cleveref}

\newcommand{\bb}{\mathbb}
\newcommand{\R}{\bb R}
\newcommand{\Z}{{\bb Z}}
\newcommand{\N}{{\bb N}}

\newcommand{\abs}[1]{\lvert#1\rvert}
\newcommand{\norm}[1]{\lVert#1\rVert}
\newcommand{\st}{\mathrel{}\middle|\mathrel{}}
\newcommand{\overbar}[1]{\mkern 1.5mu\overline{\mkern-1.5mu#1\mkern-1.5mu}\mkern 1.5mu}

\newtheorem{theorem}{Theorem}[section]
\newtheorem{proposition}[theorem]{Proposition}

\newtheorem{lemma}[theorem]{Lemma}
\newtheorem{definition}[theorem]{Definition}
\newtheorem{conjecture}[theorem]{Conjecture}
\newtheorem{observation}[theorem]{Observation}

\DeclareMathOperator{\ReLU}{ReLU}
\DeclareMathOperator{\CPWL}{CPWL}
\DeclareMathOperator{\CCPWL}{CCPWL}
\DeclareMathOperator{\MAX}{MAX}
\DeclareMathOperator{\conv}{conv}
\DeclareMathOperator{\cone}{cone}
\DeclareMathOperator{\even}{even}
\DeclareMathOperator{\odd}{odd}
\DeclareMathOperator{\Newt}{Newt}

\usetikzlibrary{arrows.meta, shapes}
\tikzset{bigneuron/.style={circle, draw, inner sep = 0, minimum width = 5.5ex}}
\tikzset{smallneuron/.style={circle, draw, inner sep = 0, minimum width = 4.5ex}}
\tikzset{transform/.style={fill=white, circle}}
\tikzset{connection/.style={-{Stealth}}}
\newcommand{\relu}{
	\begin{tikzpicture}
		\draw [line width=1pt] (-1.1ex,0) -- (0,0) -- (0.9ex,0.9ex);
	\end{tikzpicture}
}

\title{Towards Lower Bounds on the Depth\\of ReLU Neural Networks
	\thanks{Authors' accepted manuscript; to appear in the SIAM Journal on Discrete Mathematics. A preliminary conference version appeared in the proceedings of the NeurIPS 2021 conference. We thank the anonymous referees of both the journal and the conference version for their insightful comments which helped to improve the presentation and clarity.

	Christoph Hertrich gratefully acknowledges funding by DFG-GRK 2434 ``Facets of Complexity''. Amitabh Basu gratefully acknowledges support from AFOSR Grant FA95502010341 and NSF Grant CCF2006587. Martin Skutella gratefully acknowledges funding by the Deutsche Forschungsgemeinschaft (DFG, German Research Foundation) under Germany's Excellence Strategy --- The Berlin Mathematics Research Center MATH+ (EXC-2046/1, project ID: 390685689).
}}

\author{Christoph Hertrich}
\affil{\small
	London School of Economics and Political Science, London, UK,\protect\\c.hertrich@lse.ac.uk}

\author{Amitabh Basu}
	\affil{\small
	Johns Hopkins University, Baltimore, USA,\protect\\basu.amitabh@jhu.edu}

\author{Marco Di Summa}
	\affil{\small
	Universit{\`a} degli Studi di Padova, Padua, Italy,\protect\\disumma@math.unipd.it}

\author{Martin Skutella}
	\affil{\small
	Technische Universit{\"a}t Berlin, Berlin, Germany,\protect\\martin.skutella@tu-berlin.de}

\date{}

\begin{document}

\maketitle

\begin{abstract}
	We contribute to a better understanding of the class of functions that can be represented by a neural network with ReLU activations and a given architecture. Using techniques from mixed-integer optimization, polyhedral theory, and tropical geometry, we provide a mathematical counterbalance to the universal approximation theorems which suggest that a single hidden layer is sufficient for learning any function. In particular, we investigate whether the class of {\em exactly} representable functions {\em strictly} increases by adding more layers (with no restrictions on size).
	As a by-product of our investigations, we settle an old conjecture about piecewise linear functions by Wang and Sun~\cite{wang2005generalization} in the affirmative.
	We also present upper bounds on the sizes of neural networks required to represent functions with logarithmic depth.
\end{abstract}

\section{Introduction}

A core problem in machine learning and statistics is the estimation of an unknown data distribution with access to independent and identically distributed samples from the distribution. It is well-known that there is a tension between the expressivity of the model chosen to approximate the distribution and the number of samples needed to solve the problem with high confidence (or equivalently, the variance one has in one's estimate). This is referred to as the {\em bias-variance} trade-off or the {\em bias-complexity} trade-off. Neural networks provide a way to turn this bias-complexity knob in a controlled manner that has been studied for decades going back to the idea of a {\em perceptron} by Rosenblatt~\cite{rosenblatt1958perceptron}. This is done by modifying the {\em architecture} of a neural network class of functions, in particular its {\em size} in terms of {\em depth} and {\em width}. As one increases these parameters, the class of functions becomes more expressive. In terms of the bias-variance trade-off, the ``bias'' decreases as the class of functions becomes more expressive, but the ``variance'' or ``complexity'' increases.

So-called {\em universal approximation theorems}~\cite{anthony2009neural,cybenko1989approximation,hornik1991approximation} show that even with a single hidden layer, that is, when the depth of the architecture achieves its smallest possible value, one can essentially reduce the ``bias'' as much as one desires, by increasing the width. Nevertheless, it can be advantageous both theoretically and empirically to increase the depth because a substantial reduction in the size can be achieved by this~\cite{Arora:DNNwithReLU, eldan2016power, liang2017deep, safran2017depth, Telgarsky15, telgarsky2016benefits, yarotsky2017error}. To get a better quantitative handle on these trade-offs, it is important to understand what classes of functions are exactly representable by neural networks with a certain architecture. The precise mathematical statements of universal approximation theorems show that single layer networks can arbitrarily well {\em approximate} any continuous function (under some additional mild hypotheses). While this suggests that single layer networks are good enough from a learning perspective, from a mathematical perspective, one can ask the question if the class of functions represented by single layer networks is a {\em strict} subset of the class of functions represented by networks with two or more hidden layers. On the question of size, one can ask for precise bounds on the required width of a network with given depth to represent a certain class of functions. A better understanding of the function classes exactly represented by different architectures has implications not just for mathematical foundations, but also algorithmic and statistical learning aspects of neural networks, as recent advances on the training complexity show~\cite{Arora:DNNwithReLU,bertschinger2022training, GKMR21,froese2022computational,khalife2022neural}. The task of searching for the ``best'' function in a class can only benefit from a better understanding of the nature of functions in that class. A motivating question behind the results in this paper is to understand the hierarchy of function classes exactly represented by neural networks of increasing depth.

We now introduce more precise notation and terminology to set the stage for our investigations.

\subsection{Notation and Definitions} We write $[n]\coloneqq\{1,2,\dots,n\}$ for the set of natural numbers up to $n$ (without zero) and $[n]_0\coloneqq[n]\cup\{0\}$ for the same set including zero. For any $n\in\N$, let $\sigma\colon\R^n\to\R^n$ be the component-wise \emph{rectifier} function \[\sigma(x)=(\max\{0,x_1\},\max\{0,x_2\},\dots,\max\{0,x_n\}).\]

For any \emph{number of hidden layers} $k\in\N$, a \emph{$(k+1)$-layer feedforward neural network with rectified linear units} (ReLU NN or simply NN) is given by $k$ affine transformations $T^{(\ell)}\colon\R^{n_{\ell-1}}\to\R^{n_\ell}$, $x\mapsto A^{(\ell)}x+b^{(\ell)}$, for $\ell\in[k]$, and a linear transformation $T^{(k+1)}\colon\R^{n_{k}}\to\R^{n_{k+1}}$, $x\mapsto A^{(k+1)} x$. It is said to \emph{compute} or \emph{represent} the function $f\colon\R^{n_0}\to\R^{n_{k+1}}$ given by
\[
f=T^{(k+1)}\circ\sigma\circ T^{(k)}\circ\sigma\circ\dots\circ T^{(2)}\circ\sigma\circ T^{(1)}.
\]
The matrices $A^{(\ell)}\in\R^{n_{\ell}\times n_{\ell-1}}$ are called the \emph{weights} and the vectors $b^{(\ell)}\in\R^{n_\ell}$ are the \emph{biases} of the $\ell$-th layer. The number $n_\ell\in\N$ is called the \emph{width} of the $\ell$-th layer. The maximum width of all hidden layers $\max_{\ell\in[k]} n_\ell$ is called the \emph{width} of the NN. Further, we say that the NN has \emph{depth} $k+1$ and \emph{size} $\sum_{\ell=1}^k n_\ell$.

Often, NNs are represented as layered, directed, acyclic graphs where each dimension of each layer (including input layer $\ell=0$ and output layer $\ell=k+1$) is one vertex, weights are arc labels, and biases are node labels. Then, the vertices are called \emph{neurons}.

For a given input $x=x^{(0)}\in\R^{n_0}$, let $y^{(\ell)}\coloneqq T^{(\ell)}(x^{(\ell-1)})\in\R^{n_\ell}$ be the \emph{activation vector} and $x^{(\ell)}\coloneqq \sigma(y^{\ell})\in\R^{n_\ell}$ the \emph{output vector} of the $\ell$-th layer. Further, let $y\coloneqq y^{(k+1)}=f(x)$ be the \emph{output} of the NN. We also say that the $i$-th component of each of these vectors is the \emph{activation} or the \emph{output} of the $i$-th neuron in the $\ell$-th layer.

To illustrate the definition of NNs and how they compute functions, \Cref{Fig:Max2Num} shows an NN with one hidden layer computing the maximum of two numbers.

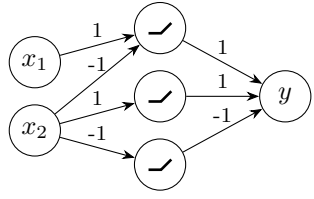
\begin{figure}[t]
	\centering
	\begin{tikzpicture}[every node/.style={transform shape}]
		\small
		\node[smallneuron] (x) at (0,6ex) {$x_1$};
		\node[smallneuron] (y) at (0,0) {$x_2$};
		\node[smallneuron] (n11) at (11ex,9ex) {\relu};
		\node[smallneuron] (n12) at (11ex,3ex) {\relu};
		\node[smallneuron] (n13) at (11ex,-3ex) {\relu};
		\node[smallneuron] (m) at (22ex,3ex) {$y$};
		\scriptsize
		\draw[connection] (n11) -- node[above]{1} (m);
		\draw[connection] (x) -- node[above]{1} (n11);
		\draw[connection] (y) -- node[above]{-1} (n11);
		\draw[connection] (n12) -- node[above]{1} (m);
		\draw[connection] (n13) -- node[above]{-1} (m);
		\draw[connection] (y) -- node[above]{1} (n12);
		\draw[connection] (y) -- node[above]{-1} (n13);
	\end{tikzpicture}
	\caption{An NN with two input neurons, labeled $x_1$ and $x_2$, three
		hidden neurons, labeled with the shape of the rectifier function, and one output neuron, labeled $y$. The arcs
		are labeled with their weights and all biases are zero. The NN has depth 2, width 3, and size 3. It
		computes the function \mbox{$x\mapsto y= \max\{0,x_1-x_2\}+\max\{0,x_2\}-\max\{0,-x_2\}= \max\{0,x_1-x_2\}+x_2=\max\{x_1,x_2\}$}.}
	\label{Fig:Max2Num}
\end{figure}

For~$k\in\N$, we define
\begin{align*}
	\ReLU_n(k)&\coloneqq\{f\colon\R^n\to\R\mid \text{$f$ can be represented by a $(k+1)$-layer NN}\},\\
	\CPWL_n&\coloneqq\{f\colon\R^n\to\R\mid \text{$f$ is continuous and piecewise linear}\}.
\end{align*}
By definition, a continuous function~$f\colon\R^n\to\R$ is piecewise linear in case there is a finite set of polyhedra whose union is~$\R^n$, and~$f$ is affine linear over each such polyhedron.

In order to analyze $\ReLU_n(k)$, we use another function class defined as follows. We call a function~$g$ a {\em $p$-term max} function if it can be expressed as maximum of $p$ affine terms, that is, $g(x)=\max\{\ell_1(x), \ldots, \ell_p(x)\}$ where $\ell_i\colon \R^n \to \R$ is affine linear for~$i\in [p]$. Note that this also includes max functions with less than $p$ terms, as some functions $\ell_{i}$ may coincide. Based on that, we define
\begin{align*}
	\MAX_n(p)&\coloneqq\{f\colon\R^n\to\R\mid \text{$f$ is a linear combination of $p$-term max functions}\}.
\end{align*}
Note that Wang and Sun~\cite{wang2005generalization} call $p$-term max functions \emph{$(p-1)$-order hinges} and linear combinations of those \emph{$(p-1)$-order hinging hyperplanes}.

If the input dimension $n$ is not important for the context, we sometimes drop the index and use $\ReLU(k)\coloneqq\bigcup_{n\in\N}\ReLU_n(k)$ and $\MAX(p)\coloneqq\bigcup_{n\in\N}\MAX_n(p)$ instead.

We will use the standard notations $\conv A$ and $\cone A$ for the convex and conic hulls of a set $A\subseteq\R^n$. For an in-depth treatment of polyhedra and (mixed-integer) optimization, we refer to the book by Schrijver~\cite{sch}.

\subsection{Representing Piecewise Linear Functions with ReLU Networks}

It is not hard to see that every function expressed by a ReLU network is continuous and piecewise linear (CPWL) because it is composed of affine transformations and ReLU functions, which are both CPWL. Based on a result by Wang and Sun~\cite{wang2005generalization}, Arora et al.~\cite{Arora:DNNwithReLU} prove that the converse is true as well by showing that any CPWL function can be represented with logarithmic depth.

\begin{theorem}[Arora et al.~\cite{Arora:DNNwithReLU}]\label{Thm:CPWL}
	If $n\in\N$ and $k^*\coloneqq\lceil\log_2 (n + 1)\rceil$, then \mbox{$\CPWL_n = \ReLU_n(k^*)$}.
\end{theorem}

Since this result is the starting point for our paper, let us briefly sketch its proof.
For this purpose, we start with a simple special case of a CPWL function: the maximum of~$n$ numbers. Recall that one hidden layer suffices to compute the maximum of two numbers, see \Cref{Fig:Max2Num}. Now one can easily stack this operation: in order to compute the maximum of four numbers, we divide them into two pairs with two numbers each, compute the maximum of each pair and then the maximum of the two results. This idea results in the NN depicted in \Cref{Fig:Max4Num}, which has two hidden layers.
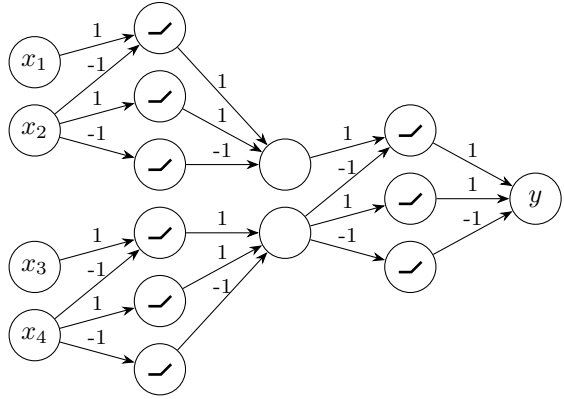
\begin{figure}[t]
	\centering
	\begin{tikzpicture}[every node/.style={transform shape}]\small
		\node[smallneuron] (x1) at (0,6ex) {$x_1$};
		\node[smallneuron] (x2) at (0,0) {$x_2$};
		\node[smallneuron] (x3) at (0,-12ex) {$x_3$};
		\node[smallneuron] (x4) at (0,-18ex) {$x_4$};
		\node[smallneuron] (n11) at (11ex,9ex) {\relu};
		\node[smallneuron] (n12) at (11ex,3ex) {\relu};
		\node[smallneuron] (n13) at (11ex,-3ex) {\relu};
		\node[smallneuron] (n14) at (11ex,-9ex) {\relu};
		\node[smallneuron] (n15) at (11ex,-15ex) {\relu};
		\node[smallneuron] (n16) at (11ex,-21ex) {\relu};
		\node[smallneuron] (m1) at (22ex,-3ex) {};
		\node[smallneuron] (m2) at (22ex,-9ex) {};
		\node[smallneuron] (n21) at (33ex,0ex) {\relu};
		\node[smallneuron] (n22) at (33ex,-6ex) {\relu};
		\node[smallneuron] (n23) at (33ex,-12ex) {\relu};
		\node[smallneuron] (m) at (44ex,-6ex) {$y$};
		\smaller
		\draw[connection] (n21) -- node[above]{1} (m);
		\draw[connection] (n22) -- node[above]{1} (m);
		\draw[connection] (n23) -- node[above]{-1} (m);
		\draw[connection] (n11) -- node[above]{1} (m1);
		\draw[connection] (n12) -- node[above]{1} (m1);
		\draw[connection] (n13) -- node[above]{-1} (m1);
		\draw[connection] (n14) -- node[above]{1} (m2);
		\draw[connection] (n15) -- node[above]{1} (m2);
		\draw[connection] (n16) -- node[above]{-1} (m2);
		\draw[connection] (x1) -- node[above]{1} (n11);
		\draw[connection] (x2) -- node[above]{-1} (n11);
		\draw[connection] (x2) -- node[above]{1} (n12);
		\draw[connection] (x2) -- node[above]{-1} (n13);
		\draw[connection] (m1) -- node[above]{1} (n21);
		\draw[connection] (m2) -- node[above]{-1} (n21);
		\draw[connection] (m2) -- node[above]{1} (n22);
		\draw[connection] (m2) -- node[above]{-1} (n23);
		\draw[connection] (x3) -- node[above]{1} (n14);
		\draw[connection] (x4) -- node[above]{-1} (n14);
		\draw[connection] (x4) -- node[above]{1} (n15);
		\draw[connection] (x4) -- node[above]{-1} (n16);
	\end{tikzpicture}
	\caption{An NN to compute the maximum of four numbers that consists of three copies of the NN in \Cref{Fig:Max2Num}. Note that no activiation function is applied at the two unlabeled middle vertices (representing $\max\{x_1,x_2\}$ and~$\max\{x_3,x_4\}$). Therefore, the linear transformations directly before and after these vertices can be combined into a single one. Thus, the network has total depth three (two hidden layers).}
	\label{Fig:Max4Num}
\end{figure}
Repeating this procedure, one can compute the maximum of eight numbers with three hidden layers, and, in general, the maximum of $2^k$ numbers with $k$ hidden layers. Phrasing this the other way around, we obtain that the maximum of $n$ numbers can be computed with $\lceil\log_2(n)\rceil$ hidden layers. Since NNs can easily form affine combinations, this implies the following lemma.

\begin{lemma}[Arora et al.~\cite{Arora:DNNwithReLU}] \label{Lem:max}
	If $n,k\in\N$, then $\MAX_n(2^k)\subseteq\ReLU_n(k)$.
\end{lemma}

The question whether the depth of this construction is best possible is one of the central open questions we attack in this paper.

In fact, the maximum function is not just a nice toy example, it is, in some sense, the most difficult one of all CPWL function to represent for a ReLU~NN. This is due to a result by Wang and Sun~\cite{wang2005generalization} stating that every CPWL function defined on $\R^n$ can be written as linear combination of $(n+1)$-term max functions.

\begin{theorem}[Wang and Sun~\cite{wang2005generalization}]\label{Thm:wang}
	If $n\in\N$, then $\CPWL_n=\MAX_n(n+1)$.
\end{theorem}

The proof given by Wang and Sun~\cite{wang2005generalization} is technically involved and we do not go into details here. However, in \Cref{sec:width} we provide an alternative proof yielding a slightly stronger result. This will be useful to bound the width of NNs representing arbitrary CPWL functions.

\Cref{Thm:CPWL} by Arora et al.~\cite{Arora:DNNwithReLU} can now be deduced from combining \Cref{Lem:max} and \Cref{Thm:wang}: In fact, for $k^*=\lceil\log_2 (n + 1)\rceil$, one obtains \[\CPWL_n=\MAX_n(n+1)\subseteq\ReLU_n(k^*)\subseteq\CPWL_n\] and thus equality in the whole chain of subset relations.

\subsection{Our Main Conjecture} 

We wish to understand whether the logarithmic depth bound in \Cref{Thm:CPWL} by Arora et al.~\cite{Arora:DNNwithReLU} is best possible or whether one can do better. We believe it is indeed best possible and pose the following conjecture to better understand the importance of depth in neural networks.

\begin{conjecture}\label{Con:main}
	For every $n\in\N$, let $k^*\coloneqq\lceil\log_2 (n + 1)\rceil$. Then it holds that
	\begin{equation}\label{Equ:chain}
		\ReLU_n(0)\subsetneq\ReLU_n(1)\subsetneq\dots\subsetneq\ReLU_n(k^*-1)\subsetneq\ReLU_n(k^*)=\CPWL_n.
	\end{equation}
\end{conjecture}

\Cref{Con:main} claims that any additional layer up to $k^*$ hidden layers strictly increases the set of representable functions. This would imply that the construction by Arora et al.~\cite{Arora:DNNwithReLU} is actually depth-minimal.

Observe that, in order to prove \Cref{Con:main}, it is sufficient to find, for every $k^* \in \N$, one function \mbox{$f\in\ReLU_n(k^*)\setminus\ReLU_n(k^*-1)$} with $n = 2^{k^* - 1}$.
This also implies all other strict inclusions $\ReLU_n(i-1)\subsetneq\ReLU_n(i)$ for $i < k^*$ because \mbox{$\ReLU_n(i-1)=\ReLU_n(i)$} immediately implies that $\ReLU_n(i-1)=\ReLU_n(i')$ for all~\mbox{$i'\geq i-1$}.

In fact, thanks to \Cref{Thm:wang} by Wang and Sun~\cite{wang2005generalization}, there is a canonical candidate for such a function, allowing us to reformulate the conjecture as follows.

\begin{conjecture}\label{Con:simplified}
	For $k\in\N$, $n=2^k$, the function $f_n(x)=\max\{0,x_1,\dots,x_n\}$ cannot be represented with $k$ hidden layers, that is, $f_n\notin \ReLU_n(k)$.
\end{conjecture}

\begin{proposition}\label{prop:equivalence}
	\Cref{Con:main} and \Cref{Con:simplified} are equivalent.
\end{proposition}
\begin{proof}
	We argued above that \Cref{Con:simplified} implies \Cref{Con:main}. For the other direction, we prove the contraposition, that is, assuming that \Cref{Con:simplified} is violated, we show that \Cref{Con:main} is violated as well. To this end, suppose there is a $k\in\N$, $n=2^k$, such that $f_n$ is representable with $k$ hidden layers. We argue that under this hypothesis, any $(n+1)$-term max function can be represented with $k$ hidden layers. To see this, observe that
	\[\max\{\ell_1(x), \ldots, \ell_{n+1}(x)\} = \max\{0, \ell_1(x)-\ell_{n+1}(x), \ldots, \ell_{n}(x)-\ell_{n+1}(x)\} + \ell_{n+1}(x).\]
	Modifying the first-layer weights of the NN computing $f_n$ such that input $x_i$ is replaced by the affine expression $\ell_{i}(x)-\ell_{n+1}(x)$, one obtains a $k$-hidden-layer NN computing the function \mbox{$\max\{0, \ell_1(x)-\ell_{n+1}(x), \ldots, \ell_{n}(x)-\ell_{n+1}(x)\}$}. Moreover, since affine functions, in particular also $\ell_{n+1}(x)$, can easily be represented by $k$-hidden-layer NNs, we obtain that any $(n+1)$-term maximum is in $\ReLU_n(k)$. Using \Cref{Thm:wang} by Wang and Sun~\cite{wang2005generalization}, it follows that $\ReLU_n(k)=\CPWL_n$. In particular, since $k^*\coloneqq\lceil\log_2 (n + 1)\rceil=k+1$, we obtain that \Cref{Con:main} must be violated as well.
\end{proof}

It is known that \Cref{Con:simplified} holds for $k=1$~\cite{mukherjee2017lower}, that is, the CPWL function~$\max\{0,x_1,x_2\}$ cannot be computed by a 2-layer NN. The reason for this is that the set of breakpoints of a CPWL function computed by a 2-layer NN is always a union of lines, while the set of breakpoints of $\max\{0,x_1,x_2\}$ is a union of three half-lines; compare \Cref{Fig:max3} and the detailed proof by Mukherjee and Basu~\cite{mukherjee2017lower}.
Moreover, in subsequent work to the first version of this article, it was shown that the conjecture is true for all $k\in\N$ if one only allows integer weights in the neural network~\cite{haase2023lower}. However, this proof does not easily generalize to arbitrary, real-valued weights. Thus, the conjecture remains open for all $k \geq 2$.

\begin{figure}[t]
	\centering
	\begin{tikzpicture}[scale=1.1]
		\draw[] (0,0) -- (1,1);
		\draw[] (0,-1.2) -- (0,0);
		\draw[] (-1.2,0) -- (0,0);
		\node at (0.6,-0.24) {$x_1$};
		\node at (-0.6,-0.6) {$0$};
		\node at (-0.24,0.6) {$x_2$};
	\end{tikzpicture}\hspace{0.1\linewidth}
	\begin{tikzpicture}[every node/.style={transform shape}]
		\small
		\node[smallneuron] (x) at (0,6ex) {$x_1$};
		\node[smallneuron] (y) at (0,0) {$x_2$};
		\node[smallneuron] (n11) at (11ex,9ex) {\relu};
		\node[rotate=90] (n12) at (11ex,3ex) {$\cdots$};
		\node[smallneuron] (n13) at (11ex,-3ex) {\relu};
		\node[smallneuron] (m) at (22ex,3ex) {$y$};
		\draw[connection] (n11) -- (m);
		\draw[connection] (x) -- (n11);
		\draw[connection] (x) -- (n13);
		\draw[connection] (n13) -- (m);
		\draw[connection] (y) -- (n11);
		\draw[connection] (y) -- (n13);
	\end{tikzpicture}\hspace{0.1\linewidth}
	\begin{tikzpicture}[scale=1.1]
		\fill[black!10] (-1.1,-1.1) rectangle (1.1,1.1);
		\draw[] (-1.1,-0.8) -- (1.1,-0.1);
		\draw[] (0.6,1.1) -- (-0.8,-1.1);
		\draw[] (-0.9,1.1) -- (1.1,-1.0);
		\draw[] (-1.1,-0.3) -- (1.1,0.9);
	\end{tikzpicture}
	\caption{Set of breakpoints of the function $\max\{0,x_1,x_2\}$ (left). This function cannot be computed by a 2-layer NN (middle), since the set of breakpoints of any function computed by such an NN is always a union of lines (right).}
	\label{Fig:max3}
\end{figure}
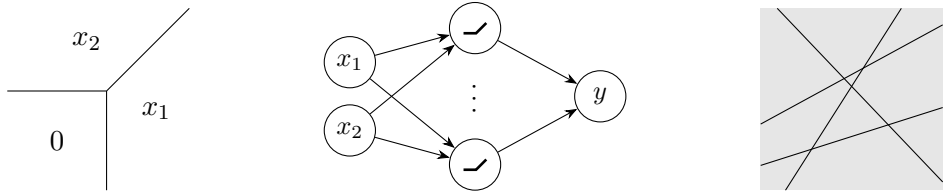

\subsection{Contribution and Outline}

In this paper, we present the following results as partial progress towards resolving this conjecture.

In \Cref{Sec:MIP}, we resolve \Cref{Con:simplified} for $k=2$, under a natural assumption on the breakpoints of the function represented by any intermediate neuron. Intuitively, the assumption states that no neuron introduces unexpected breakpoints compared to the final function we want to represent. We call such neural networks \emph{$H$-conforming}, see \Cref{Sec:MIP} for a formal definition. We then provide a computer-based proof leveraging techniques from mixed-integer programming for the following theorem.

\begin{restatable}{theorem}{thmmaxwithassumption}\label{Thm:maxWithAssumption}
	There does not exist an $H$-conforming 3-layer ReLU NN computing the function $\max\{0,x_1,x_2,x_3,x_4\}$.
\end{restatable}

In the light of \Cref{Lem:max}, stating that $\MAX(2^k)\subseteq\ReLU(k)$ for all $k\in\N$, one might ask whether the converse is true as well, that is, whether the classes $\MAX(2^k)$ and $\ReLU(k)$ are actually equal. This would not only provide a neat characterization of $\ReLU(k)$, but also prove \Cref{Con:simplified} without any additional assumption since one can show that $\max\{0,x_1,\dots,x_{2^k}\}$ is not contained in $\MAX(2^k)$.

In fact, for $k=1$, it is true that $\ReLU(1)=\MAX(2)$, that is, a function is computable with one hidden layer if and only if it is a linear combination of 2-term max functions.	
However, in \Cref{sec:more}, we show the following theorem.

\begin{restatable}{theorem}{thmricher}\label{thm:richer}
	For every $k\geq2$, the set $\ReLU(k)$ is a strict superset of $\MAX(2^k)$.
\end{restatable}

To achieve this result, the key technical ingredient is the theory of polyhedral complexes associated with CPWL functions. This way, we provide important insights concerning the richness of the class $\ReLU(k)$. As a by-product, the results in \Cref{sec:more} imply that $\MAX_n(n)$ is a strict subset of $\CPWL_n=\MAX_n(n+1)$, which was conjectured by Wang and Sun~\cite{wang2005generalization} in 2005, but has been open since then.

So far, we have focused on understanding the smallest depth needed to express CPWL functions using neural networks with ReLU activations. In \Cref{sec:width}, we complement these results by upper bounds on the sizes of the networks needed for expressing arbitrary CPWL functions. In particular, we show the following theorem.

\begin{restatable}{theorem}{thmwidthbound}\label{thm:width-bound}
	Let $f\colon\R^n\to\R$ be a CPWL function with $p$ affine pieces. Then $f$ can be represented by a ReLU NN with depth $\lceil\log_2(n+1)\rceil+1$ and width $\mathcal{O}(p^{2n^2+3n+1})$.
\end{restatable}

We arrive at this result by introducing a novel application of recently established interrelations between neural networks and tropical geometry.

\Cref{thm:width-bound} improves upon a previous bound by He et al.~\cite{he2020relu} because it is polynomial in $p$ if $n$ is regarded as fixed constant, while the bounds in \cite{he2020relu} are exponential in $p$.
In subsequent work to the first version of our article, it was shown that the width of the network can be drastically decreased if one allows more depth (in the order of $\log (p)$ instead of $\log (n)$)~\cite{chen2022improved}.

Let us remark that there are different definitions of the \emph{number of pieces} $p$ of a CPWL function $f$ in the literature, compare the discussions in~\cite{chen2022improved,he2020relu} about \emph{pieces} versus \emph{linear components}. Our bounds work with any of these definitions since they apply to the smallest possible way to define $p$, called \emph{linear components} in~\cite{chen2022improved}: for our purposes, $p$ can be defined as the smallest number of affine functions such that, at each point, $f$ is equal to one of these affine functions. Since all other definitions of the \emph{number of pieces} are at least that large, our bounds are valid for these definitions as well.

Finally, in \Cref{Sec:Polytopes}, we provide an outlook how these interactions between tropical geometry and NNs could possibly also be useful to provide a full, unconditional proof of \Cref{Con:main} by means of polytope theory. This yields another equivalent rephrasing of \Cref{Con:main} which is stated purely in the language of basic operations on polytopes and does not involve neural networks any more.

We conclude in \Cref{Sec:discuss} with a discussion of further open research questions.

\subsection{Further Related Work}\label{Sec:Lit}\begin{samepage}
	\paragraph{Depth versus size} Soon\end{samepage} after the original universal approximation theorems~\cite{cybenko1989approximation,hornik1991approximation}, concrete bounds were obtained on the number of neurons needed in the hidden layer to achieve a certain level of accuracy. The literature on this is vast and we refer to a small representative sample here~\cite{barron1993universal,barron1994approximation,mhaskar1993approximation,pinkus1999approximation,mhaskar1996neural,mhaskar1995degree}. More recent research has focused on how deeper networks can have exponentially or super exponentially smaller size compared to shallower networks~\cite{vardi2021size,Arora:DNNwithReLU, eldan2016power, hanin2019universal, hanin2017approximating, liang2017deep,  nguyen2018neural, raghu2017expressive, safran2017depth, Telgarsky15, telgarsky2016benefits, yarotsky2017error}. See also~\cite{gribonval2021approximation} for another perspective on the relationship between expressivity and architecture, and the references therein. 

\paragraph{Mixed-integer optimization and machine learning} Over the past decade, a growing body of work has emerged that explores the interplay between mixed-integer optimization and machine learning. On the one hand, researchers have attempted to improve mixed-integer optimization algorithms by exploiting novel techniques from machine learning~\cite{bonami2018learning,gasse2019exact,he2014learning,khalil2016learning,khalil2017learning,kruber2017learning,lodi2017learning,alvarez2017machine}; see also~\cite{bengio2020machine} for a recent survey. On the flip side, mixed-integer optimization techniques have been used to analyze function classes represented by neural networks~\cite{serra2018bounding,anderson2020strong,fischetti2017deep,serra2020empirical,serra2020lossless}.
In \Cref{Sec:MIP} below, we show another new use of mixed-integer optimization tools for understanding function classes represented by neural networks.

\paragraph{Design of training algorithms} We believe that a better understanding of the function classes represented exactly by a neural architecture also has benefits in terms of understanding the complexity of the training problem. For instance, in work by Arora et al.~\cite{Arora:DNNwithReLU}, an understanding of single layer ReLU networks enables the design of a globally optimal algorithm for solving the empirical risk minimization (ERM) problem, that runs in polynomial time in the number of data points in fixed dimension. See also~\cite{goel2017reliably,goel2018learning,goel2019learning,dey2020approximation,boob2022complexity,GKMR21,froese2022computational,abrahamsen2021training,bienstock2018principled,bertschinger2022training,Chen2022_LearningFTP,khalife2022neural} for similar lines of work.

\paragraph{Neural Networks and Tropical Geometry}
A recent stream of research involves the interplay between neural networks and tropical geometry. The piecewise linear functions computed by neural networks can be seen as (tropical quotients of) tropical polynomials. Linear regions of these functions correspond to vertices of so-called \emph{Newton polytopes} associated with these tropical polynomials. Applications of this correspondence include bounding the number of linear regions of a neural network~\cite{Zhang:Tropical, charisopoulos2018tropical, montufar2022sharp} and understanding decision boundaries \cite{alfarra2022decision}. In \Cref{sec:width} we present a novel application of tropical concepts to understand neural networks.
We refer to \cite{maragos2021tropical} for a recent survey of connections between machine learning and tropical geometry, as well as to the textbooks by Maclagan and Sturmfels~\cite{maclagan2015introduction} and Joswig~\cite{ETC} for in-depth introductions to tropical geometry and tropical combinatorics.

\section{Conditional Lower Depth Bounds via Mixed-Integer Programming}\label{Sec:MIP}

In this section, we provide a computer-aided proof that, under a natural, yet unproven assumption, the function $f(x)\coloneqq\max\{0,x_1,x_2,x_3,x_4\}$ cannot be represented by a 3-layer NN. It is worth to note that, to the best of our knowledge, no CPWL function is known for which the non-existence of a 3-layer NN can be proven without additional assumptions. For ease of notation, we write $x_0\coloneqq0$.

We first prove that we may restrict ourselves to NNs without biases. This holds true independent of our assumption, which we introduce afterwards.

\begin{definition}
	A function $g\colon\R^n\to\R^m$ is called \emph{positively homogeneous} if it satisfies~$g(\lambda x) = \lambda g(x)$ for all~$\lambda\ge 0$.
\end{definition}

\begin{definition}
	For an NN given by transformations $T^{(\ell)}(x)=A^{(\ell)}x+b^{(\ell)}$, we define the corresponding \emph{homogenized NN} to be the NN given by $\tilde{T}^{(\ell)}(x)=A^{(\ell)}x$ with all biases set to zero.
\end{definition}

\begin{proposition}\label{prop:wlognobias}
	If an NN computes a positively homogeneous function, then the corresponding homogenized NN computes the same function.
\end{proposition}
\begin{proof}
	Let $g\colon\R^{n_0}\to\R^{n_{k+1}}$ be the function computed by the original NN and $\tilde{g}$ the one computed by the homogenized NN.
	Further, for any $0\leq\ell\leq k$, let \[g^{(\ell)}=T^{(\ell+1)}\circ\sigma\circ T^{(\ell)}\circ\dots\circ T^{(2)}\circ\sigma\circ T^{(1)}\] be the function computed by the sub-NN consisting of the first $(\ell+1)$-layers and let~$\tilde{g}^{(\ell)}$ be the function computed by the corresponding homogenized sub-NN.
	We first show by induction on $\ell$ that the norm of~$\norm{g^{(\ell)}(x)-\tilde{g}^{(\ell)}(x)}$ is bounded by a global constant that only depends on the parameters of the NN but not on~$x$.
	
	For $\ell=0$, we have $\norm{g^{(0)}(x)-\tilde{g}^{(0)}(x)}=\norm{b^{(1)}}\eqqcolon C_0$, settling the induction base. For the induction step, let $\ell\geq1$ and assume that $\norm{g^{(\ell-1)}(x)-\tilde{g}^{(\ell-1)}(x)}\leq C_{\ell-1}$, where $C_{\ell-1}$ only depends on the parameters of the NN. Since a component-wise application of the ReLU activation function has Lipschitz constant 1, this implies~\mbox{$\norm{(\sigma\circ g^{(\ell-1)})(x)-(\sigma\circ\tilde{g}^{(\ell-1)})(x)}\leq C_{\ell-1}$}. Using the spectral matrix norm $\norm{A}$ of a matrix $A$, we obtain:
	\begin{align*}
		\norm{g^{(\ell)}(x)-\tilde{g}^{(\ell)}(x)}~&=~\norm{b^{(\ell+1)} + A^{(\ell+1)}((\sigma\circ g^{(\ell-1)})(x)-(\sigma\circ\tilde{g}^{(\ell-1)})(x))}\\
		&\leq~\norm{b^{(\ell+1)}} + \norm{A^{(\ell+1)}}\cdot C_{\ell-1}~\eqqcolon~C_{\ell}
	\end{align*}
	Since the right-hand side only depends on NN parameters, the induction is completed.
	
	Finally, we show that $g=\tilde{g}$. For the sake of contradiction, suppose that there is an~\mbox{$x\in\R^{n_0}$} with~$\norm{g(x)-\tilde{g}(x)}=\delta>0$. Let $x'\coloneqq \frac{C_k+1}{\delta}x$; then, by positive homogeneity of $g$ (by assumption) and $\tilde{g}$ (by construction and because the ReLU function is positively homogeneous), it follows that $\norm{g(x')-\tilde{g}(x')}=C_k+1>C_k$, contradicting the property shown above. Thus, we have $g=\tilde{g}$.
\end{proof}

Since $f=\max\{0,x_1,x_2,x_3,x_4\}$ is positively homogeneous, \Cref{prop:wlognobias} implies that, if there is a 3-layer NN computing $f$, then there also is one that has no biases. Therefore, in the remainder of this section, we only consider NNs without biases and assume implicitly that all considered CPWL functions are positively homogeneous. In particular, any piece of such a CPWL function is linear and not only affine linear.

Observe that, for the function $f$, the only points of non-differentiability (a.k.a.\ \emph{breakpoints}) are at places where at least two of the five numbers $x_0=0$, $x_1$, $x_2$, $x_3$, and~$x_4$ are equal. Hence, if some neuron of an NN computing $f$ introduces breakpoints at other places, these breakpoints must be canceled out by other neurons. Therefore, we find it natural to work under the assumption that such breakpoints need not be introduced at all in the first place. 

To make this assumption formal, let $H_{ij}=\{x\in\R^4\mid x_i=x_j\}$, for \mbox{$0\leq i < j \leq 4$}, be ten hyperplanes in $\R^4$ and $H=\bigcup_{0\leq i < j \leq 4} H_{ij}$ be the corresponding hyperplane arrangement.
This is the intersection of the so-called \emph{braid arrangement} in five dimensions with the hyperplane $x_0=0$~\cite{stanley2004introduction}.
The \emph{regions} or \emph{cells} of $H$ are defined to be the closures of the connected components of $\R^4\setminus H$. It is easy to see that these regions are in one-to-one correspondence to the $5!=120$ possible orderings of the five numbers $x_0=0$, $x_1$, $x_2$, $x_3$, and $x_4$. More precisely, for a permutation $\pi$ of the five indices $[4]_0=\{0,1,2,3,4\}$, the corresponding region is the polyhedron  
\[C_\pi~\coloneqq~\{x\in\R^4\mid x_{\pi(0)}\leq x_{\pi(1)}\leq x_{\pi(2)}\leq x_{\pi(3)}\leq x_{\pi(4)}\}.\]

\begin{definition}
	We say that a (positively homogeneous) CPWL function $g$ is \emph{$H$-conforming}, if it is linear within any of these regions of $H$, that is, if it only has breakpoints where the relative ordering of the five values $x_0=0$, $x_1$, $x_2$, $x_3$, $x_4$ changes. Moreover, an NN is said to be \emph{$H$-conforming} if the output of each neuron contained in the NN is $H$-conforming.
\end{definition}
See \Cref{Fig:assumption} for an illustration of the definition in the (simpler) two-dimensional case. Note that, by the definition, an NN is $H$-conforming if and only if, for all layers $\ell\in[k]$, the intermediate function \mbox{$\sigma\circ T^{(\ell)}\circ\sigma\circ T^{(\ell-1)}\circ\dots\circ\sigma\circ T^{(1)}$} is $H$-conforming.

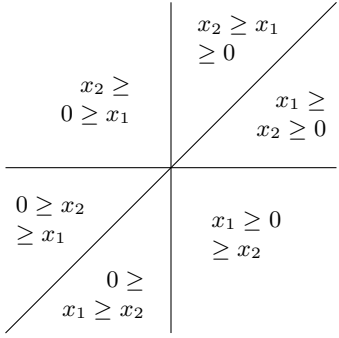
\begin{figure}[t]
	\centering
	\begin{tikzpicture}
		\draw[] (-2.2,-2.2) -- (2.2,2.2);
		\draw[] (0,-2.2) -- (0,2.2);
		\draw[] (-2.2,0) -- (2.2,0);
		\footnotesize
		\node[align=right] at (1.6,0.7) {$x_1\geq$\\$x_2\geq 0$};
		\node[align=left] at (-1.6,-0.7) {$0\geq x_2$\\$\geq x_1$};
		\node[align=left] at (1,-0.9) {$x_1\geq 0$\\$\geq x_2$};
		\node[align=right] at (-1,0.9) {$x_2\geq$\\$0\geq x_1$};
		\node[align=left] at (0.9,1.7) {$x_2\geq x_1$\\$\geq 0$};
		\node[align=right] at (-0.9,-1.7) {$0\geq$\\$x_1\geq x_2$};
	\end{tikzpicture}
	\caption{A function is $H$-conforming if the set of breakpoints is a subset of the hyperplane arrangement $H$. The arrangement $H$ consists of all hyperplanes where two of the coordinates (possibly including $x_0=0$) are equal. Here,~$H$ is illustrated for the (simpler) two-dimensional case, where it consists of three hyperplanes that divide the space into six cells.}
	\label{Fig:assumption}
\end{figure}

As argued above, it is plausible that considering $H$-conforming NNs is enough to prove \Cref{Con:main}. In other words, we conjecture that, if there exists a 3-layer NN computing the function $f(x)=\max\{0,x_1,x_2,x_3,x_4\}$, then there also exists one that is $H$-conforming. This motivates the following theorem, which we prove computer-aided by means of mixed-integer programming.

\thmmaxwithassumption*

The remainder of this section is devoted to proving this theorem. The rough outline of the proof is as follows. We first study some geometric properties of the hyperplane arrangement $H$. This will show that each of the $120$ cells of $H$ is a simplicial polyhedral cone spanned by $4$ extreme rays. In total, there are $30$ such rays (because rays are used multiple times to span different cones). This implies that each $H$-conforming function is uniquely determined by its values on the $30$ rays and, therefore, the set of $H$-conforming functions of type $\R^4\to\R$ is a $30$-dimensional vector space. We then use linear algebra to show that the space of functions generated by $H$\=/conforming two-layer NNs is a $14$-dimensional subspace. Moreover, with two hidden layers, at least $29$ of the~$30$ dimensions can be generated and $f$ is not contained in this $29$-dimensional subspace. So the remaining question is whether the $14$ dimensions producible with the first hidden layer can be combined in such a way that after applying a ReLU activation in the second hidden layer, we do not end up within the $29$-dimensional subspace. We model this question as a mixed-integer program (MIP). Solving the MIP yields that we always end up within the $29$-dimensional subspace, implying that $f$ cannot be represented by a 3-layer NN. This provides a computational proof of \Cref{Thm:maxWithAssumption}.

Let us start with investigating the structure of the hyperplane arrangement~$H$. For readers familiar with the interplay between hyperplane arrangements and polytopes, it is worth noting that~$H$ is dual to a combinatorial equivalent of the 4\=/dimensional permutahedron. Hence, what we are studying in the following are some combinatorial properties of the permutahedron.

Recall that the regions of $H$ are given by the 120 polyhedra  
\[C_\pi~\coloneqq~\{x\in\R^4\mid x_{\pi(0)}\leq x_{\pi(1)}\leq x_{\pi(2)}\leq x_{\pi(3)}\leq x_{\pi(4)}\}\subseteq\R^4\]
for each permutation $\pi$ of $[4]_0$, where $x_0$ is used as a replacement for $0$.
With this representation, one can see that $C_\pi$ is a pointed polyhedral cone (with the origin as its only vertex) spanned by the four half-lines (a.k.a.~\emph{rays})

\begin{align*}
	R_{\{\pi(0)\}}~&\coloneqq~ \{x\in\R^4\mid x_{\pi(0)}\leq x_{\pi(1)}= x_{\pi(2)}= x_{\pi(3)}= x_{\pi(4)}\},\\
	R_{\{\pi(0),\pi(1)\}}~&\coloneqq~ \{x\in\R^4\mid x_{\pi(0)}= x_{\pi(1)}\leq x_{\pi(2)}= x_{\pi(3)}= x_{\pi(4)}\},\\
	R_{\{\pi(0),\pi(1),\pi(2)\}}~&\coloneqq~ \{x\in\R^4\mid x_{\pi(0)}= x_{\pi(1)}= x_{\pi(2)}\leq x_{\pi(3)}= x_{\pi(4)}\},\\
	R_{\{\pi(0),\pi(1),\pi(2),\pi(3)\}}~&\coloneqq~ \{x\in\R^4\mid x_{\pi(0)}= x_{\pi(1)}= x_{\pi(2)}= x_{\pi(3)}\leq x_{\pi(4)}\}.
\end{align*}

Observe that these objects are indeed rays anchored at the origin because the three equalities define a one-dimensional subspace of $\R^4$ and the inequality cuts away one of the two directions.

With that notation, we see that each of the $120$ cells of $H$ is a simplicial cone spanned by four out of the $30$ rays $R_S$ with $\emptyset\subsetneq S\subsetneq[4]_0$. For each such set $S$, denote its complement by $\bar{S}\coloneqq[4]_0\setminus S$. Let us use a generating vector $r_S\in \R^4$ for each of these rays such that $R_S=\cone r_S$ as follows: If $0\in S$, then $r_S\coloneqq\mathds{1}_{\bar{S}}\in\R^4$, otherwise $r_S\coloneqq-\mathds{1}_{S}\in\R^4$, where for each $S\subseteq[4]$, the vector $\mathds{1}_S\in\R^4$ contains entries~$1$ at precisely those index positions that are contained in $S$ and entries $0$ elsewhere. For example, $r_{\{0,2,3\}}=(1,0,0,1)\in\R^4$ and $r_{\{1,4\}}=(-1,0,0,-1)\in\R^4$. Then, the set~$R$ containing conic generators of all the $30$ rays of $H$ consists of the $30$ vectors \mbox{$R=(\{0,1\}^4\cup\{0,-1\}^4)\setminus\{0\}^4$}.

Let $\mathcal{S}^{30}$ be the space of all $H$-conforming CPWL functions of type $\R^4\to\R$. We show that $\mathcal{S}^{30}$ is a $30$-dimensional vector space.

\begin{lemma}\label{lem:iso}
	The map $g\mapsto (g(r))_{r\in R}$ that evaluates a function~$g\in\mathcal{S}^{30}$ at the $30$ rays in $R$ is an isomorphism between $\mathcal{S}^{30}$ and $\R^{30}$. In particular, $\mathcal{S}^{30}$ is a $30$\=/dimensional vector space. 
\end{lemma}
\begin{proof}
	First note that $\mathcal{S}^{30}$ is closed under addition and scalar multiplication. Therefore, it is a subspace of the vector space of continuous functions of type $\R^4\to\R$, and thus, in particular, a vector space.
	We show that the map $g\mapsto (g(r))_{r\in R}$ is in fact a vector space isomorphism. The map is obviously linear, so we only need to show that it is a bijection. In order to do so, remember that~$\R^4$ is the union of the $5!=120$ simplicial cones $C_\pi$. In particular, given the function values on the extreme rays of these cones, there is a unique positively homogeneous, continuous continuation that is linear within each of the 120 cones. This implies that the considered map is a bijection between~$\mathcal{S}^{30}$ and~$\R^{30}$.
\end{proof}

The previous lemma also provides a canonical basis of the vector space $\mathcal{S}^{30}$: the one consisting of all CPWL functions attaining value $1$ at one ray $r\in R$ and value~$0$ at all other rays. However, it turns out that for our purposes it is more convenient to work with a different basis. To this end, let $g_M(x)=\max_{i\in M} x_i$ for each \mbox{$M\subseteq[4]_0$} with $M\notin\{\emptyset,\{0\}\}$. These~$30$ functions contain, among other functions, the four (linear) coordinate projections \mbox{$g_{\{i\}}(x)=x_i$}, $i\in[4]$, and the function $f(x)=g_{[4]_0}(x)=\max\{0,x_1,x_2,x_3,x_4\}$.

\begin{lemma}
	The $30$ functions $g_M(x)=\max_{i\in M} x_i$ with $\{\emptyset,\{0\}\}\not\ni M\subseteq[4]_0$ form a basis of $\mathcal{S}^{30}$.
\end{lemma}
\begin{proof}
	Evaluating the $30$ functions $g_M$ at all $30$ rays $r\in R$ yields $30$ vectors in~$\R^{30}$. It can be easily verified (e.g., using a computer) that these vectors form a basis of $\R^{30}$. Thus, due to the isomorphism of \Cref{lem:iso}, the functions $g_M$ form a basis of~$\mathcal{S}^{30}$.
\end{proof}

Next, we focus on particular subspaces of $\mathcal{S}^{30}$ generated by only some of the $30$ functions~$g_M$. We prove that they correspond to the spaces of functions computable by $H$-conforming $2$- and $3$-layer NNs, respectively.

To do so, let $\mathcal{B}^{14}$ be the set of the $14$ basis functions $g_M$ with~\mbox{$\{\emptyset,\{0\}\}\not\ni M\subseteq[4]_0$} and $\abs{M}\leq 2$. Let $\mathcal{S}^{14}$ be the $14$-dimensional subspace spanned by $\mathcal{B}^{14}$. Similarly, let~$\mathcal{B}^{29}$ be the set of the $29$ basis functions $g_M$ with $\{\emptyset,\{0\}\}\not\ni M\subsetneq[4]_0$ (all but $[4]_0$). Let~$\mathcal{S}^{29}$ be the $29$-dimensional subspace spanned by $\mathcal{B}^{29}$. 

\begin{lemma}\label{lem:14}
	The space $\mathcal{S}^{14}$ consists of all functions computable by $H$-conforming $2$-layer NNs.
\end{lemma}
\begin{proof}
	Each function in $\mathcal{S}^{14}$ is a linear combination of $2$-term max functions by definition. Hence, by \Cref{Lem:max}, it can be represented by a 2-layer NN.
	
	Conversely, we show that any function representable by a 2-layer NN is indeed contained in~$\mathcal{S}^{14}$. It suffices to show that the output of every neuron in the first (and only) hidden layer of an $H$-conforming ReLU~NN is in $\mathcal{S}^{14}$ because the output of a 2-layer NN is a linear combination of such outputs. Let $a\in\R^4$ be the first-layer weights of such a neuron, computing the function $g_a(x)\coloneqq\max\{a^T x, 0\}$, which has the hyperplane $\{x\in\R^4\mid a^T x = 0\}$ as breakpoints (or is constantly zero). Since the NN must be $H$-conforming, this must be one of the ten hyperplanes $x_i = x_j$, $0\leq i<j\leq 4$. Thus, $g_a(x)=\max\{\lambda (x_i - x_j), 0\}$ for some $\lambda\in\R$. If $\lambda \geq 0$, it follows that $g_a = \lambda g_{\{i,j\}} - \lambda g_{\{j\}}\in \mathcal{S}^{14}$, and if $\lambda \leq 0$, we obtain $g_a = -\lambda g_{\{i,j\}} + \lambda g_{\{i\}}\in \mathcal{S}^{14}$. This concludes the proof.
\end{proof}

For 3-layer NNs, an analogous statement can be made. However, only one direction can be easily seen.

\begin{lemma} 
	Any function in $\mathcal{S}^{29}$ can be represented by an $H$-conforming $3$-layer NN.
\end{lemma}
\begin{proof}
	As in the previous lemma, each function in $\mathcal{S}^{29}$ is a linear combination of $4$-term max functions by definition. Hence, by \Cref{Lem:max}, it can be represented by a 3-layer NN.
\end{proof}

Our goal is to prove the converse as well: any $H$-conforming function represented by a 3-layer NN is in $\mathcal{S}^{29}$. Since $f(x)=\max\{0,x_1,x_2,x_3,x_4\}$ is the 30th basis function, which is linearly independent from $\mathcal{B}^{29}$ and thus not contained in $\mathcal{S}^{29}$, this implies \Cref{Thm:maxWithAssumption}. To achieve this goal, we first provide another characterization of $\mathcal{S}^{29}$, which can be seen as an orthogonal direction to $\mathcal{S}^{29}$ in $\mathcal{S}^{30}$. For a function $g\in\mathcal{S}^{30}$, let
\[
\phi(g)\coloneqq \sum_{\emptyset\subsetneq S\subsetneq[4]_0} (-1)^{\abs{S}} g(r_S)
\]
be a linear map from $\mathcal{S}^{30}$ to $\R$.

\begin{lemma}\label{lem:orthogonal}
	A function $g\in\mathcal{S}^{30}$ is contained in $\mathcal{S}^{29}$ if and only if $\phi(g)=0$.
\end{lemma}
\begin{proof}
	Any $g\in\mathcal{S}^{30}$ can be represented as a unique linear combination of the $30$ basis functions $g_M$ and is contained in $\mathcal{S}^{29}$ if and only if the coefficient of $f=g_{[4]_0}$ is zero.
	One can easily check (with a computer) that $\phi$ maps all functions in $\mathcal{B}^{29}$ to $0$, but not the 30th basis function $f$. Thus, $g$ is contained in $\mathcal{S}^{29}$ if and only if it satisfies $\phi(g)=0$.
\end{proof}

In order to make use of our assumption that the NN is $H$-conforming, we need the following insight about when the property of being $H$-conforming is preserved after applying a ReLU activation.

\begin{lemma}\label{lem:conforming}
	Let $g\in\mathcal{S}^{30}$. The function $h=\sigma\circ g$ is $H$-conforming (and thus in~$\mathcal{S}^{30}$ as well) if and only if there is no pair of sets $\emptyset\subsetneq S\subsetneq S' \subsetneq[4]_0$ with $g(r_{S})$ and~$g(r_{S'})$ being nonzero and having different signs.
\end{lemma}
\begin{proof}
	The key observation to prove this lemma is the following: for two rays $r_S$ and~$r_{S'}$, there exists a cell $C$ of the hyperplane arrangement $H$ for which both $r_S$ and~$r_{S'}$ are extreme rays if and only if $S\subsetneq S'$ or $S'\subsetneq S$.
	
	Hence, if there exists a pair of sets $\emptyset\subsetneq S\subsetneq S' \subsetneq[4]_0$ with $g(r_{S})$ and $g(r_{S'})$ being nonzero and having different signs, then the function $g$ restricted to $C$ is a linear function with both strictly positive and strictly negative values. Therefore, after applying the ReLU activation, the resulting function $h$ has breakpoints within~$C$ and is not $H$-conforming.
	
	Conversely, if for each pair of sets $\emptyset\subsetneq S\subsetneq S' \subsetneq[4]_0$, both $g(r_{S})$ and $g(r_{S'})$ are either nonpositive or nonnegative, then $g$ restricted to any cell $C$ of $H$ is either nonpositive or nonnegative everywhere. In the first case, $h$ restricted to that cell $C$ is the zero function, while in the second case, $h$ coincides with $g$ in $C$. In both cases,~$h$ is linear within all cells and, thus, $H$-conforming.
\end{proof}

Having collected all these lemmas, we are finally able to construct an MIP whose solution proves that any function computed by an $H$-conforming 3-layer NN is in $\mathcal{S}^{29}$. As in the proof of \Cref{lem:14}, it suffices to focus on the output of a single neuron in the second hidden layer. Let $h=\sigma\circ g$ be the output of such a neuron with $g$ being its input. Observe that, by construction, $g$ is a function computed by a $2$-layer NN, and thus, by \Cref{lem:14}, a linear combination of the $14$ functions in~$\mathcal{B}^{14}$. The MIP contains three types of variables, which we denote in bold to distinguish them from constants: 
\begin{itemize}
	\item $14$ continuous variables $\mathbf{a}_M\in[-1,1]$, being the coefficients of the linear combination of the basis of $\mathcal{S}^{14}$ forming $g$, that is, $g=\sum_{g_M\in\mathcal{B}^{14}} \mathbf{a}_M g_M$ (since multiplying~$g$ and $h$ with a nonzero scalar does not alter the containment of~$h$ in $\mathcal{S}^{29}$, we may restrict the variables to $[-1,1]$),
	\item $30$ binary variables $\mathbf{z}_S \in\{0,1\}$ for $\emptyset\subsetneq S\subsetneq[4]_0$, determining whether the considered neuron is strictly active at ray $r_S$, that is, whether $g(r_S)>0$,
	\item $30$ continuous variables $\mathbf{y}_S\in\R$ for $\emptyset\subsetneq S\subsetneq[4]_0$, representing the output of the considered neuron at all rays, that is, $\mathbf{y}_S=h(r_S)$.
\end{itemize}

To ensure that these variables interact as expected, we need two types of constraints:
\begin{itemize}
	\item For each of the $30$ rays $r_S$, $\emptyset\subsetneq S\subsetneq[4]_0$, the following constraints ensure that~$\mathbf{z}_S$ and output $\mathbf{y}_S$ are correctly calculated from the variables $\mathbf{a}_M$, that is, $\mathbf{z}_S=1$ if and only if~$g(r_S)=\sum_{g_{M}\in\mathcal{B}^{14}} \mathbf{a}_{M} g_{M} (r_S)$ is positive, and $\mathbf{y}_S=\max\{0, g(r_S)\}$. Also compare the references given in \Cref{Sec:Lit} concerning MIP models for ReLU units. Note that the restriction of the coefficients $\mathbf{a}_{M}$ to $[-1,1]$ ensures that the absolute value of $g(r_S)$ is always bounded by $14$, allowing us to use $15$ as a replacement for $+\infty$:
	\begin{align}
		\begin{split}\label{eq:MIPrelu}
			\mathbf{y}_S &\geq 0\\
			\mathbf{y}_S &\geq \sum_{g_{M}\in\mathcal{B}^{14}} \mathbf{a}_{M} g_{M} (r_S)\\
			\mathbf{y}_S &\leq 15 \mathbf{z}_S\\
			\mathbf{y}_S &\leq \sum_{g_{M}\in\mathcal{B}^{14}} \mathbf{a}_{M} g_{M} (r_S) + 15(1-\mathbf{z}_S)
		\end{split}
	\end{align}
	Observe that these constraints ensure that one of the following two cases occurs: If~$\mathbf{z}_S=0$, then the first and third line imply $\mathbf{y}_S=0$ and the second line implies that the incoming activation is in fact nonpositive. The fourth line is always satisfied in that case. Otherwise, if $\mathbf{z}_S=1$, then the second and fourth line imply that $\mathbf{y}_S$ equals the incoming activation, and, in combination with the first line, this has to be nonnegative. The third line is always satisfied in that case. Hence, the set of constraints \eqref{eq:MIPrelu} correctly models the ReLU activation function.
	\item For each of the $150$ pairs of sets $\emptyset\subsetneq S\subsetneq S' \subsetneq[4]_0$, the following constraints ensure that the property in \Cref{lem:conforming} is satisfied. More precisely, if one of the variables~$\mathbf{z}_{S}$ or $\mathbf{z}_{S'}$ equals $1$, then the ray of the other set has nonnegative activation, that is, $g(r_{S'})\geq 0$ or  $g(r_{S})\geq 0$, respectively:
	\begin{align}
		\begin{split}			
			\label{eq:MIPconforming}
			\sum_{g_{M}\in\mathcal{B}^{14}} \mathbf{a}_{M} g_{M} (r_S)&\geq 15 (\mathbf{z}_{S'}-1)\\
			\sum_{g_{M}\in\mathcal{B}^{14}} \mathbf{a}_{M} g_{M} (r_{S'})&\geq 15 (\mathbf{z}_S-1)
		\end{split}
	\end{align}
	Observe that these constraints successfully prevent that the two rays $r_S$ and~$r_{S'}$ have nonzero activations with different signs. Conversely, if this is not the case, then we can always satisfy constraints \eqref{eq:MIPconforming} by setting only those variables $\mathbf{z}_S$ to value $1$ where the activation of ray $r_S$ is \emph{strictly} positive. (Note that, if the incoming activation is precisely zero, constraints \eqref{eq:MIPrelu} make it possible to choose both values $0$ or $1$ for $\mathbf{z}_S$.) Hence, these constraints are in fact appropriate to model $H$-conformity.
\end{itemize}

In the light of \Cref{lem:orthogonal}, the objective function of our MIP is to maximize~$\phi(h)$, that is, the expression
\[
\sum_{\emptyset\subsetneq S\subsetneq[4]_0} (-1)^{\abs{S}} \mathbf{y}_S.
\]

The MIP has a total of 30 binary and 44 continuous variables, as well as 420 inequality constraints.  The next proposition formalizes how this MIP can be used to check whether a 3-layer NN function can exist outside $\mathcal{S}^{29}$.

\begin{proposition}\label{prop:MIP}
	There exists an $H$-conforming 3-layer NN computing a function not contained in~$\mathcal{S}^{29}$ if and only if the objective value of the MIP defined above is strictly positive.
\end{proposition}
\begin{proof}
	For the first direction, assume that such an NN exists. Since its final output is a linear combination of the outputs of the neurons in the second hidden layer, one of these neurons must compute a function $\tilde{h}=\sigma\circ \tilde{g}\notin\mathcal{S}^{29}$, with $\tilde{g}$ being the input to that neuron. By \Cref{lem:orthogonal}, it follows that $\phi(\tilde{h})\neq 0$. Moreover, we can even assume without loss of generality that $\phi(\tilde{h})> 0$, as we argue now. If this is not the case, multiply all first-layer weights of the NN by~$-1$ to obtain a new NN computing function $\hat{h}$ instead of $\tilde{h}$.
	Observing that $r_S=-r_{[4]_0\setminus S}$ for all $r_S\in R$, we obtain~$\hat{h}(r_S)=\tilde{h}(-r_S)=\tilde{h}(r_{[4]_0\setminus S})$ for all $r_S\in R$. Plugging this into the definition of $\phi$ and using that the cardinalities of~$S$ and~$[4]_0\setminus S$ have different parity, we further obtain $\phi(\hat{h})=-\phi(\tilde{h})$. Therefore, we can assume that $\phi(\tilde{h})$ was already positive in the first place.
	
	Using \Cref{lem:14}, the function $\tilde{g}$ can be represented as a linear combination $\tilde{g}=\sum_{g_M\in\mathcal{B}^{14}} \mathbf{\tilde{a}}_M g_M$ of the functions in $\mathcal{B}^{14}$. Let $\alpha\coloneqq \max_{M} \abs{\mathbf{\tilde{a}}_M}$. Note that $\alpha>0$ because otherwise $\tilde{g}$ would be the zero function. Let us define modified functions $g$ and $h$ from $\tilde{g}$ and $\tilde{h}$ as follows. Let $\mathbf{a}_M\coloneqq \mathbf{\tilde{a}}_M/\alpha\in[-1,1]$, $g\coloneqq\sum_{g_M\in\mathcal{B}^{14}} \mathbf{a}_M g_M$, and $h\coloneqq\sigma\circ g$. Moreover, for all rays $r_S\in R$, let $\mathbf{y}_S\coloneqq h(r_S)$, as well as $\mathbf{z}_S\coloneqq 1$ if $\mathbf{y}_S>0$, and $\mathbf{z}_S\coloneqq 0$ otherwise.
	
	It is easy to verify that the variables $\mathbf{a}_M$, $\mathbf{y}_S$, and $\mathbf{z}_S$ defined that way satisfy \eqref{eq:MIPrelu}. Moreover, since the NN is $H$-conforming, they also satisfy \eqref{eq:MIPconforming}. Finally, they also yield a strictly positive objective function value since $\phi(h) = \phi(\tilde{h})/\alpha>0$.
	
	For the reverse direction, assume that there exists an MIP solution consisting of~$\mathbf{a}_M$,~$\mathbf{y}_S$, and~$\mathbf{z}_S$, satisfying \eqref{eq:MIPrelu} and \eqref{eq:MIPconforming}, and having a strictly positive objective function value. Define the functions~\mbox{$g\coloneqq\sum_{g_M\in\mathcal{B}^{14}} \mathbf{a}_M g_M$} and $h\coloneqq\sigma\circ g$. One concludes from \eqref{eq:MIPrelu} that $h(r_S)=\mathbf{y}_S$ for all rays~$r_S\in R$. \Cref{lem:14} implies that~$g$ can be represented by a 2-layer NN. Thus, $h$ can be represented by a 3-layer NN. Moreover, constraints~\eqref{eq:MIPconforming} guarantee that this NN is $H$-conforming. Finally, since the MIP solution has strictly positive objective function value, we obtain $\phi(h)>0$, implying that~$h\notin\mathcal{S}^{29}$.
\end{proof}

In order to use the MIP as part of a mathematical proof, we employed an MIP solver that uses exact rational arithmetics without numerical errors, namely the solver by the Parma Polyhedral Library (PPL) \cite{BagnaraHZ08SCP}. We called the solver from a SageMath (Version 9.0) \cite{sagemath} script on a machine with an Intel Core i7-8700 6-Core 64-bit CPU and 15.5 GB RAM, using the openSUSE Leap 15.2 Linux distribution. SageMath, which natively includes the PPL solver, is published under the GPLv3 license. After a total running time of almost 7 days (153 hours), we obtained optimal objective function value zero. This makes it possible to prove \Cref{Thm:maxWithAssumption}.

\begin{proof}[Proof of \Cref{Thm:maxWithAssumption}]
	Since the MIP has optimal objective function value zero, \Cref{prop:MIP} implies that any function computed by an $H$-conforming $3$-layer NN is contained in $\mathcal{S}^{29}$. In particular, it is not possible to compute the function \mbox{$f(x)=\max\{0,x_1,x_2,x_3,x_4\}$} with an $H$-conforming $3$-layer~NN.
\end{proof}

We remark that state-of-the-art MIP solver Gurobi (version 9.1.1) \cite{gurobi}, which is commercial but offers free academic licenses, is able to solve the same MIP within less than a second, providing the same result. However, Gurobi does not employ exact arithmetics, making it impossible to exclude numerical errors and use it as a mathematical proof.

The SageMath code can be found on GitHub at
\begin{center}
	\url{https://github.com/ChristophHertrich/relu-mip-depth-bound}.
\end{center}
Additionally, the MIP can be found there as~\texttt{.mps} file, a standard format to represent MIPs. This allows one to use any solver of choice to reproduce our result.

\section{Going Beyond Linear Combinations of Max Functions}\label{sec:more}

In this section we prove the following result, showing that NNs with $k$ hidden layers can compute more functions than only linear combinations of $2^k$-term max functions.

\thmricher*

In order to prove this theorem, for each number of hidden layers $k\geq2$, we provide a specific function in~\mbox{$\ReLU(k)\setminus \MAX(2^k)$}. The challenging part is to show that the function is in fact not contained in $\MAX(2^k)$.

\begin{proposition}\label{prop:main}
	For any $n\geq 3$, the function $f\colon\R^n\to\R$ defined by
	\begin{equation}\label{eq:f}
		f(x)=\max\{0,x_1, x_2,\dots,x_{n-3},\,\max\{x_{n-2},x_{n-1}\} + \max\{0,x_n\}\}
	\end{equation}
	is not contained in  $\MAX(n)$.
\end{proposition}

This means that $f$ cannot be written as a linear combination of $n$-term max functions, which proves a conjecture by \cite{wang2005generalization} that $\MAX_n(n)\subsetneq\CPWL_n$, which has been open since 2005. Previously, it was only known that linear combinations of~$(n-1)$-term maxes are not sufficient to represent any CPWL function defined on~$\R^n$, that is, $\MAX_n(n-1)\subsetneq\CPWL_n$. Lu~\cite{lu2021note} provides a short analytical argument for this fact. 

Before we prove \Cref{prop:main}, we show that it implies \Cref{thm:richer}.

\begin{proof}[Proof of \Cref{thm:richer}]
	For $k\geq2$, let $n\coloneqq2^k$. By \Cref{prop:main}, function $f$ defined in \eqref{eq:f} is not contained in $\MAX(2^k)$. It remains to show that it can be represented using a ReLU NN with $k$ hidden layers. To see this, first observe that any of the $n/2=2^{k-1}$ terms $\max\{0,x_1\}$, $\max\{x_{2i},x_{2i+1}\}$ for $i\in[n/2-2]$, and $\max\{x_{n-2},x_{n-1}\} + \max\{0,x_n\}$ can be expressed by a one-hidden-layer NN since all these are (linear combinations of) $2$-term max functions. Since $f$ is the maximum of these $2^{k-1}$ terms, and since the maximum of $2^{k-1}$ numbers can be computed with $k-1$ hidden layers (\Cref{Lem:max}), this implies that $f$ is in $\ReLU(k)$.
\end{proof}

In order to prove \Cref{prop:main}, we need the concept of polyhedral complexes. A \emph{polyhedral complex} $\mathcal P$ is a finite set of polyhedra such that each face of a polyhedron in~$\mathcal P$ is also in $\mathcal P$, and for two polyhedra $P,Q\in\mathcal{P}$, their intersection $P\cap Q$ is a common face of~$P$ and~$Q$ (possibly the empty face).
Given a polyhedral complex $\mathcal P$ in $\R^n$ and an integer $m\in[n]$, we let $\mathcal P^m$ denote the collection of all $m$-dimensional polyhedra in $\mathcal P$.

For a convex CPWL function $f$, we define its \emph{underlying polyhedral complex} as follows: it is the unique polyhedral complex covering $\R^n$ (i.e., each point in $\R^n$ belongs to some polyhedron in $\mathcal P$) whose $n$-dimensional polyhedra coincide with the domains of the (maximal) affine pieces of $f$. In particular, $f$ is affine linear within each $P\in\mathcal{P}$, but not within any strict superset of a polyhedron in $\mathcal P^n$.

Exploiting properties of polyhedral complexes associated with CPWL functions, we prove the following proposition below.

\begin{proposition}\label{prop:main2}
	Let $f_0\colon\R^n\to\R$ be a convex CPWL function and let $\mathcal P_0$ be the underlying polyhedral complex. If there exists a hyperplane $H\subseteq\R^n$ such that the set
	\[T\coloneqq\bigcup\left\{F\in\mathcal P_0^{n-1}\st F\subseteq H\right\}\]
	is nonempty and contains no line, then $f_0$ cannot be expressed as a linear combination of $n$-term maxima of affine linear functions.
\end{proposition}

Again, before we proceed to the proof of \Cref{prop:main2}, we show that it implies \Cref{prop:main}.

\begin{proof}[Proof of \Cref{prop:main}]
	Observe that $f$ (defined in \eqref{eq:f}) has the alternate representation
	\[
	f(x)=\max\{0,\, x_1,\, x_2,\, \dots,\, x_{n-3},\, x_{n-2},\, x_{n-1},\, x_{n-2} + x_n,\, x_{n-1} + x_n\}
	\]
	as a maximum of $n+2$ terms.
	Let $\mathcal{P}$ be its underlying polyhedral complex. 
	Let the hyperplane $H$ be defined by $x_1=0$.
	
	Observe that any facet in $\mathcal{P}^{n-1}$ is a polyhedron defined by two of the $n+2$ terms that are equal and at least as large as each of the remaining $n$ terms. Hence, the only facet that could possibly be contained in $H$ is
	\[F\coloneqq\{x\in\R^n\mid x_1=0 \geq x_2,\, \dots,\, x_{n-3},\, x_{n-2},\, x_{n-1},\, x_{n-2} + x_n,\, x_{n-1} + x_n\}.\]
	
	Note that $F$ is indeed an $(n-1)$-dimensional facet in $\mathcal{P}^{n-1}$, because, for example, a small ball around $(0,-1,\dots,-1)\in \R^n$ intersected with $H$ is contained in $F$.
	
	Finally, we need to show that $F$ is pointed, that is, it contains no line. A well-known fact from polyhedral theory says if there is any line in $F$ with direction~\mbox{$d\in \R^n\setminus\{0\}$}, then $d$ must satisfy the defining inequalities with equality. However, only the zero vector does this. Hence, $F$ cannot contain a line.
	
	Therefore, when applying \Cref{prop:main2} to $f$ with underlying polyhedral complex~$\mathcal{P}$ and hyperplane~$H$, we have $T=F$, which is nonempty and contains no line. Hence, $f$ cannot be written as linear combination of $n$-term maxima.
\end{proof}

The remainder of this section is devoted to proving \Cref{prop:main2}.
In order to exploit properties of the underlying polyhedral complex of the considered CPWL functions, we will first introduce some terminology, notation, and results related to polyhedral complexes in $\R^n$ for any $n\ge1$.

\begin{definition}
	Given an abelian group $(G,+)$, we define $\mathcal F^n(G)$ as the family of all functions $\phi$ of the form $\phi\colon\mathcal P^n\to G$, where $\mathcal P$ is a polyhedral complex that covers~$\R^n$. We say that $\mathcal P$ is the \emph{underlying} polyhedral complex, or the polyhedral complex \emph{associated} with $\phi$.
\end{definition}

Just to give an intuition of the reason for this definition, let us mention that later we will choose $(G,+)$ to be the set of affine linear maps $\R^n\to\R$ with respect to the standard operation of sum of functions. Moreover, given a convex CPWL function $f\colon\R^n\to\R$ with underlying polyhedral complex $\mathcal P$, we will consider the following function $\phi\in\mathcal F^n(G)$: for every $P\in\mathcal P^n$, $\phi(P)$ will be the affine linear map that coincides with $f$ over $P$. It can be helpful, though not necessary, to keep this in mind when reading the next definitions and observations.

It is useful to observe that the functions in $\mathcal F^n(G)$ can also be described in a different way.
Before explaining this, we need to define an ordering between the two elements of each pair of opposite halfspaces. 
More precisely, let $H$ be a hyperplane in~$\R^n$ and let~$H',H''$ be the two closed halfspaces delimited by $H$. We choose an arbitrary rule to say that $H'$ ``precedes'' $H''$, which we write as $H'\prec H''$.\footnote{In case one wants to see such a rule explicitly, this is a possible way: Fix an arbitrary $\bar x\in H$. We can say that $H'\prec H''$ if and only if $\bar x+e_i\in H'$, where $e_i$ is the first vector in the standard basis of $\R^d$ that does not lie on $H$ (i.e., $e_1,\dots,e_{i-1}\in H$ and $e_i\notin H$). Note that this definition does not depend on the choice of $\bar x$.} We can then extend this ordering rule to those pairs of $n$-dimensional polyhedra of a polyhedral complex in~$\R^n$ that share a facet. Specifically, given a polyhedral complex~$\mathcal P$ in $\R^n$, let $P',P''\in\mathcal P^n$ be such that $F\coloneqq P'\cap P''\in\mathcal P^{n-1}$. Further, let $H$ be the unique hyperplane containing $F$. We say that $P'\prec P''$ if the halfspace delimited by $H$ and containing $P'$ precedes the halfspace delimited by $H$ and containing $P''$.

We can now explain the alternate description of the functions in $\mathcal F^n(G)$, which is based on the following notion.

\begin{definition}
	Let $\phi\in\mathcal F^n(G)$, with associated polyhedral complex $\mathcal P$. The {\em facet-function} associated with $\phi$ is the function $\psi\colon\mathcal P^{n-1}\to G$ defined as follows: given \mbox{$F\in \mathcal P^{n-1}$}, let $P',P''$ be the two polyhedra in $\mathcal P^n$ such that $F=P'\cap P''$, where $P'\prec P''$; then we set $\psi(F)\coloneqq\phi(P')-\phi(P'')$.
\end{definition}

Although it will not be used, we observe that knowing $\psi$ is sufficient to reconstruct~$\phi$ up to an additive constant. This means that a function $\phi'\in\mathcal F^n(G)$ associated with the same polyhedral complex $\mathcal P$ has the same facet-function $\psi$ if and only if there exists~\mbox{$g\in G$} such that $\phi(P)-\phi'(P)=g$ for every $P\in\mathcal P^n$. (However, it is not true that every function~\mbox{$\psi\colon\mathcal P^{n-1}\to G$} is the facet-function of some function in $\mathcal F^n(G)$.)

We now introduce a sum operation over $\mathcal F^n(G)$.

\begin{definition}\label{def:sum}
	For functions $\phi_1,\dots,\phi_p\in\mathcal F^n(G)$ with associated polyhedral complexes $\mathcal P_1,\dots,\mathcal P_p$, the sum $\phi\coloneqq\phi_1+\dots+\phi_p$ is the function in $\mathcal F^n(G)$ defined as follows:
	\begin{itemize}
		\item the polyhedral complex associated with $\phi$ is \[\mathcal P\coloneqq\{P_1\cap\dots\cap P_p\mid P_i\in\mathcal P_i\mbox{ for every $i$}\};\]
		\item given $P\in\mathcal P^n$, $P$ can be uniquely obtained as $P_1\cap\dots\cap P_p$, where $P_i\in\mathcal P^n_i$ for every $i$; we then define
		\[\phi(P)=\sum_{i=1}^p\phi_i(P_i).\]
	\end{itemize}
\end{definition}

The term ``sum'' is justified by the fact that when $\mathcal P_1=\dots=\mathcal P_p$ (and thus $\phi_1,\dots,\phi_p$ have the same domain) we obtain the standard notion of the sum of functions.

The next results shows how to compute the facet-function of a sum of functions in~$\mathcal F^n(G)$.

\begin{observation}\label{obs:psi}
	With the notation of \Cref{def:sum}, let $\psi_1,\dots,\psi_p$ be the facet-functions associated with $\phi_1,\dots,\phi_p$, and let $\psi$ be the facet-function associated with $\phi$. Given $F\in\mathcal P^{n-1}$, let $I$ be the set of indices $i\in\{1,\dots,p\}$ such that $\mathcal P_i^{n-1}$ contains a (unique) element $F_i$ with $F\subseteq F_i$. Then
	\begin{equation}\label{eq:psi-sum}
		\psi(F)=\sum_{i\in I}\psi_i(F_i).
	\end{equation}
\end{observation}

\begin{proof}
	Let $P',P''$ be the two polyhedra in $\mathcal P^n$ such that $F=P'\cap P''$, with $P'\prec P''$. We have $P'=P'_1\cap\dots\cap P'_p$ and $P''=P''_1\cap\dots\cap P''_p$ for a unique choice of $P'_i,P''_i\in\mathcal P_i^n$ for every $i$. Then
	\begin{equation}\label{eq:psi}
		\psi(F)=\phi(P')-\phi(P'')=\sum_{i=1}^p(\phi_i(P'_i)-\phi_i(P''_i)).
	\end{equation}
	Now fix $i\in[p]$. Since $F\subseteq P'_i\cap P''_i$, $\dim(P'_i\cap P''_i)\ge n-1$. If $\dim(P'_i\cap P''_i)=n-1$, then $F_i\coloneqq P'_i\cap P''_i\in\mathcal P^{n-1}_i$ and $\phi_i(P'_i)-\phi_i(P''_i)=\psi_i(F_i)$. Furthermore, $i\in I$ because~$F\subseteq F_i$.
	If, on the contrary, $\dim(P'_i\cap P''_i)=n$, the fact that $\mathcal P_i$ is a polyhedral complex implies that $P'_i=P''_i$, and thus $\phi_i(P'_i)-\phi_i(P''_i)=0$. Moreover, in this case $i\notin I$: this is because $P'\cup P''\subseteq P'_i$, which implies that the relative interior of $F$ is contained in the relative interior of $P'_i$. With these observations, from \eqref{eq:psi} we obtain \eqref{eq:psi-sum}.
\end{proof}

\begin{definition}\label{def:refinement}
	Fix $\phi\in\mathcal F^n(G)$, with associated polyhedral complex $\mathcal P$. Let $H$ be a hyperplane in $\R^n$, and let $H',H''$ be the closed halfspaces delimited by $H$. Define the polyhedral complex
	\[\widehat{\mathcal P}=\{P\cap H\mid P\in\mathcal P\}\cup\{P\cap H'\mid P\in\mathcal P\}\cup\{P\cap H''\mid P\in\mathcal P\}.\]
	The \emph{refinement} of $\phi$ with respect to $H$ is the function $\widehat\phi\in\mathcal F^n(G)$ with associated polyhedral complex $\widehat {\mathcal P}$ defined as follows: given $\widehat P\in\widehat{\mathcal P}^n$, $\widehat\phi(\widehat P)\coloneqq\phi(P)$, where $P$ is the unique polyhedron in $\mathcal P$ that contains $\widehat P$.
\end{definition}

The next results shows how to compute the facet-function of a refinement.

\begin{observation}\label{obs:hat-psi}
	With the notation of \Cref{def:refinement}, let $\psi$ be the facet-function associated with~$\phi$. Then, the facet-function $\widehat\psi$ associated with $\widehat\phi$ is given by
	\[\widehat\psi(\widehat F)=
	\begin{cases}
		\psi(F) & \mbox{if there exists a (unique) $F\in\mathcal P^{n-1}$ containing $\widehat F$}\\
		0 &\mbox{otherwise},
	\end{cases}\]
	for every~$\widehat F\in\widehat{\mathcal P}^{n-1}$.
\end{observation}

\begin{proof}
	Let $\widehat P',\widehat P''$ be the polyhedra in $\widehat{\mathcal P}^n$ such that $\widehat F=\widehat P'\cap\widehat P''$, with $\widehat P'\prec\widehat P''$. Further, let $P',P''$ be the unique polyhedra in $\mathcal P^n$ that contain $\widehat P',\widehat P''$ (respectively). It might happen that $P'=P''$.
	
	If there is $F\in\mathcal P^{n-1}$ containing $\widehat F$, then the fact that $\mathcal P$ is a polyhedral complex implies that $F=P'\cap P''$. Note that $P'\neq P''$ and $P'\prec P''$ in this case. Thus $\widehat\psi(\widehat F)=\widehat\phi(\widehat P')-\widehat\phi(\widehat P'')=\phi(P')-\phi(P'')=\psi(F)$.
	
	Assume now that no element of $\mathcal P^{n-1}$ contains $\widehat F$. Then there exists $P\in\mathcal P^n$ such that $\widehat F=P\cap H$ and $H$ intersects the interior of $P$. Note that $P=P'=P''$ in this case. Then $\widehat P'=P\cap H'$ and $\widehat P''=P\cap H''$ (or vice versa). It follows that $\widehat\psi(\widehat F)=\widehat\phi(\widehat P')-\widehat\phi(\widehat P'')=\phi(P)-\phi(P)=0$.
\end{proof}

We now prove that the operations of sum and refinement commute: the refinement of a sum is the sum of the refinements.

\begin{observation}\label{obs:hat-phi}
	Let $\phi_1,\dots,\phi_p\in\mathcal F^n(G)$ be $p$ functions with associated polyhedral complexes $\mathcal P_1,\dots,\mathcal P_p$.
	Define $\phi\coloneqq\phi_1+\dots+\phi_p$.
	Let $H$ be a hyperplane in $\R^n$, and let $H',H''$ be the closed halfspaces delimited by $H$. Then $\widehat\phi=\widehat\phi_1+\dots+\widehat\phi_p$.
\end{observation}

\begin{proof}
	Define $\widetilde\phi\coloneqq\widehat\phi_1+\dots+\widehat\phi_p$.
	It can be verified that $\widehat\phi$ and $\widetilde\phi$ are defined on the same poyhedral complex, which we denote by $\widehat P$.
	We now fix $\widehat P\in\widehat{\mathcal P}^n$ and show that $\widehat\phi(\widehat P)=\widetilde\phi(\widehat P)$.
	
	Since $\widehat P\in\widehat{\mathcal P}^n$, it is $n$-dimensional and either contained in $H'$ or $H''$. Since both cases are symmetric, let us focus on $\widehat{P}\subseteq H'$. This means, we can write it as $\widehat P=P_1\cap\dots\cap P_p\cap H'$, where $P_i\in\mathcal P_i^n$ for every $i$. Then 
	\[\widehat\phi(\widehat P)=\phi(P_1\cap\dots\cap P_p)=\sum_{i=1}^p\phi_i(P_i)=\sum_{i=1}^p\widehat\phi_i(P_i\cap H')=\widetilde\phi(P_1\cap\dots\cap P_p\cap H')=\widetilde\phi(P),\]
	where the first and third equations follow from the definition of refinement, while the second and fourth equations follow from the definition of the sum.
\end{proof}

The \emph{lineality space} of a (nonempty) polyhedron $P=\{x\in\R^n\mid Ax\leq b\}$ is the null space of the constraint matrix $A$. In other words, it is the set of vectors $y\in\R^n$ such that for every $x\in P$ the whole line~$\{x+\lambda y\mid \lambda\in\R\}$ is a subset of $P$. We say that the lineality space of $P$ is \emph{trivial}, if it contains only the zero vector, and \emph{nontrivial} otherwise.

Given a polyhedron $P$, it is well-known that all nonempty faces of $P$ share the same lineality space. Therefore, given a polyhedral complex $\mathcal P$ that covers $\R^n$, all the nonempty polyhedra in~$\mathcal P$ share the same lineality space $L$. We will call $L$ the lineality space of $\mathcal P$.

\begin{lemma}\label{lemma:main}
	Given an abelian group $(G,+)$, pick $\phi_1,\dots,\phi_p\in\mathcal F^n(G)$, with associated polyhedral complexes $\mathcal P_1,\dots,\mathcal P_p$. Assume that for every $i\in[p]$ the lineality space of $\mathcal P_i$ is nontrivial. Define $\phi\coloneqq\phi_1+\dots+\phi_p$, $\mathcal P$ as the underlying polyhedral complex, and $\psi$ as the facet-function of $\phi$. Then for every hyperplane $H\subseteq\R^n$, the set
	\[S\coloneqq\bigcup\left\{F\in\mathcal P^{n-1}\mid F\subseteq H,\,\psi(F)\ne0\right\}\]
	is either empty or contains a line.
\end{lemma}

\begin{proof}
	The proof is by induction on $n$. For $n=1$, the assumptions imply that all~$\mathcal P_i$ are equal to~$\mathcal P$, and each of these polyhedral complexes has $\R$ as its only nonempty face. Since $\mathcal P^{n-1}$ is empty, no hyperplane $H$ such that $S\ne\emptyset$ can exist.
	
	Now fix $n\ge2$. Assume by contradiction that there exists a hyperplane $H$ such that~$S$ is nonempty and contains no line. 
	Let $\widehat\phi$ be the refinement of $\phi$ with respect to~$H$, $\widehat{\mathcal P}$ be the underlying polyhedral complex, and $\widehat\psi$ be the associated facet-function. Further, we define $\mathcal Q\coloneqq\{P\cap H\mid P\in\widehat{\mathcal P}\}$, which is a polyhedral complex that covers~$H$. Note that if $H$ is identified with $\R^{n-1}$ then we can think of $\mathcal Q$ as a polyhedral complex that covers $\R^{n-1}$, and the restriction of $\widehat\psi$ to $\mathcal Q^{n-1}$, which we denote by $\phi'$, can be seen as a function in $\mathcal F^{n-1}(G)$. We will prove that $\phi'$ does not satisfy the lemma, contradicting the inductive hypothesis.
	
	Since $\phi=\phi_1+\dots+\phi_p$, by \Cref{obs:hat-phi} we have $\widehat\phi=\widehat\phi_1+\dots+\widehat\phi_p$.
	Note that for every $i\in[p]$ the hyperplane $H$ is covered by the elements of $\widehat{\mathcal P}^{n-1}$. This implies that for every $\widehat F\in\widehat{\mathcal P}^{n-1}$ and $i\in[p]$ there exists $\widehat F_i\in\widehat{\mathcal P}^{n-1}_i$ such that $\widehat F\subseteq\widehat F_i$. Then, by \Cref{obs:psi}, $\widehat\psi(\widehat F)=\widehat\psi_1(\widehat F_1)+\dots+\widehat\psi_p(\widehat F_p)$.
	
	Now, additionally suppose that $\widehat{F}$ is contained in $H$, that is, $\widehat{F}\in\mathcal{Q}^{n-1}$.
	Let~$i\in[p]$ be such that the lineality space of $\mathcal P_i$ is not a subset of the linear space parallel to~$H$. Then no element of~$\mathcal P_i^{n-1}$ contains $\widehat F_i$. By \Cref{obs:hat-psi}, $\widehat\psi_i(\widehat F_i)=0$. We then conclude that
	\[\widehat\psi(\widehat F)=\sum_{i\in J}\widehat\psi_i(\widehat F_i)\quad\mbox{for every $\widehat F\in{\mathcal Q}^{n-1}$},\]
	where $J$ is the set of indices $i$ such that the lineality space of $\mathcal P_i$ is a subset of the linear space parallel to $H$. This means that
	\[\phi'=\sum_{i\in J}\phi'_i,\]
	where $\phi'_i$ is the restriction of $\widehat\psi_i$ to $\mathcal Q_i^{n-1}$, with $\mathcal Q_i\coloneqq\{P\cap H\mid P\in\widehat{\mathcal P_i}\}$.
	Note that for every $i\in J$ the lineality space of $\mathcal Q_i$ is clearly nontrivial, as it coincides with the lineality space of $\mathcal P_i$.
	
	Now pick any $\widehat F\in\mathcal Q^{n-1}$. Note that if there exists $F\in\mathcal P^{n-1}$ such that $\widehat F\subseteq F$, then~$\widehat F=F$. It then follows from \Cref{obs:hat-psi} that
	\[\bigcup\left\{\widehat F\in\mathcal Q^{n-1}\st\widehat\psi(\widehat F)\ne0\right\}=S.\]
	In other words, 
	\begin{equation}\label{eq:S}
		\bigcup\left\{F\in\mathcal Q^{n-1}\st\phi'(F)\ne0\right\}=S.
	\end{equation}
	
	Since $S\ne H$ (as $S$ contains no line), there exists a polyhedron $F\in\mathcal Q^{n-1}$ such that~$F\subseteq S$ and $F$ has a facet $F_0$ which does not belong to any other polyhedron in~$\mathcal Q^{n-1}$ contained in $S$. Then the facet-function $\psi'$ associated with $\phi'$ satisfies~$\psi'(F_0)\ne0$. Let~$H'$ be the $(n-2)$-dimensional affine space containing $F_0$. Then the set
	\[S'\coloneqq\bigcup\left\{F\in\mathcal Q^{n-2}\st F\subseteq H',\,\psi'(F)\ne0\right\}\]
	is nonempty, as $F_0\subseteq S'$. Furthermore, we claim that $S'$ contains no line. To see why this is true, take any $F\in\mathcal Q^{n-2}$ such that $F\subseteq H'$ and $\psi'(F)\ne0$, and let $F',F''$ be the two polyhedra in $\mathcal Q^{n-1}$ having $F$ as facet. Then $\phi'(F')\ne\phi'(F'')$, and thus at least one of these values (say $\phi'(F')$) is nonzero. Then, by \eqref{eq:S}, $F'\subseteq S$, and thus also $F\subseteq S$. This shows that $S'\subseteq S$ and therefore $S'$ contains no line. 
	
	We have shown that $\phi'$ does not satisfy the lemma. This contradicts the inductive assumption that the lemma holds in dimension $n-1$.
\end{proof}

Finally, we can use this lemma to prove \Cref{prop:main2}.

\begin{proof}[Proof of \Cref{prop:main2}]
	Assume for the sake of a contradiction that
	\[f_0(x)=\sum_{i=1}^p\lambda_i\max\{\ell_{i1}(x),\dots,\ell_{in}(x)\}\quad\mbox{for every $x\in\R^n$},\]
	where $p\in\N$, $\lambda_1,\dots,\lambda_p\in\R$ and $\ell_{ij}\colon\R^n\to\R$ is an affine linear function for every~\mbox{$i\in[p]$} and~$j\in[n]$.
	Define $f_i(x)\coloneqq\lambda_i\max\{\ell_{i1}(x),\dots,\ell_{in}(x)\}$ for every $i\in[p]$, which is a CPWL function.
	
	Fix any $i\in[p]$ such that $\lambda_i\ge0$. Then $f_i$ is convex. Note that its epigraph \[E_i\coloneqq\{(x,z)\in\R^n\times\R\mid z\ge \ell_{ij}(x)\mbox{ for $j\in[n]$}\}\] is a polyhedron in $\R^{n+1}$ defined by $n$ inequalities, and thus has nontrivial lineality space. Furthermore, no line orthogonal to the $x$-space is contained in $E_i$. Since the underlying polyhedral complex $\mathcal P_i$ of $f_i$ consists of the orthogonal projections of the faces of $E_i$ (excluding $E_i$ itself) onto the $x$-space, this implies that $\mathcal P_i$ has also nontrivial lineality space. (More precisely, the lineality space of $\mathcal P_i$ is the projection of the lineality space of $E_i$.)
	
	If $\lambda_i<0$, then $f_i$ is concave. By arguing as above on the convex function $-f_i$, one obtains that the underlying polyhedral complex $\mathcal P_i$ has again nontrivial lineality space. Thus this property holds for every $i\in[p]$.
	
	The set of affine linear functions $\R^n\to\R$ forms an abelian group (with respect to the standard operation of sum of functions), which we denote by $(G,+)$. 
	For every~$i\in[p]_0$, let $\phi_i$ be the function in $\mathcal F^n(G)$ with underlying polyhedral complex $\mathcal P_i$ defined as follows: for every $P\in\mathcal P_i^n$, $\phi_i(P)$ is the affine linear function that coincides with $f_i$ over $P$. Define $\phi\coloneqq\phi_1+\dots+\phi_p$ and let $\mathcal P$ be the underlying polyhedral complex.
	
	Note that for every $P\in\mathcal P^n$, $\phi(P)$ is precisely the affine linear function that coincides with $f_0$ within~$P$.
	However, $\mathcal P$ may not coincide with $\mathcal P_0$, as there might exist~\mbox{$P',P''\in\mathcal P^d$} sharing a facet such that $\phi(P')=\phi(P'')$; when this happens, $f_0$ is affine linear over~\mbox{$P'\cup P''$} and therefore $P'$ and $P''$ are merged together in $\mathcal P_0$. Nonetheless, $\mathcal P$ is a refinement of $\mathcal P_0$, i.e., for every $P\in\mathcal P_0^n$ there exist $P_1,\dots,P_k\in\mathcal P^n$ (for some $k\ge1$) such that $P=P_1\cup\dots\cup P_k$. Moreover, $\phi_0(P)=\phi(P_1)=\dots=\phi(P_k)$. Denoting by $\psi$ the facet-function associated with $\phi$, this implies for a facet $F\in\mathcal P^{n-1}$ that $\psi(F)=0$ if and only if $F$ is not subset of any facet $F'\in\mathcal P_0^{n-1}$.
	
	Let $H$ be a hyperplane as in the statement of the proposition. The above discussion shows that 
	\[T=\bigcup\left\{F\in\mathcal P_0^{n-1}\st F\subseteq H\right\}=\bigcup\left\{F\in\mathcal P^{n-1}\st F\subseteq H,\,\psi(F)\ne0\right\}.\]
	Using $S\coloneqq T$, we obtain a contradiction to \Cref{lemma:main}.
\end{proof}

\section{A Width Bound for Neural Networks with Small Depth}\label{sec:width}

While the proof of \Cref{Thm:CPWL} by Arora et al.~\cite{Arora:DNNwithReLU} shows that \[\CPWL_n = \ReLU_n(\lceil\log_2(n+1)\rceil),\] it does not provide any bound on the width of the NN required to represent any particular CPWL function. The purpose of this section is to prove that for fixed dimension $n$, the required width for exact, depth-minimal representation of a CPWL function can be polynomially bounded in the number $p$ of affine pieces; specifically by $p^{\mathcal{O}(n^2)}$. This improves previous bounds by He et al.~\cite{he2020relu} and is closely related to works that bound the number of linear pieces of an NN as a function of the size \cite{montufar2014regions,pascanu2014number,raghu2017expressive,montufar2022sharp}.
It can also be seen as a counterpart, in the context of exact representations, to quantitative universal approximation theorems that bound the number of neurons required to achieve a certain approximation guarantee; see, e.g.,~\cite{barron1993universal,barron1994approximation,pinkus1999approximation,mhaskar1996neural,mhaskar1995degree}.

\subsection{The Convex Case}

We first derive our result for the case of convex CPWL functions and then use this to also prove the general nonconvex case. 
Our width bound is a consequence of the following theorem about convex CPWL functions, for which we are going to provide a geometric proof later.

\begin{theorem}\label{Thm:main}
	Let $f(x) = \max\{a_i^T x + b_i \mid i\in[p]\}$ be a convex CPWL function with $p$ pieces defined on $\R^n$. Then $f$ can be written as \[f(x) = \sum_{\substack{S\subseteq[p],\\\abs{S}\leq n+1}} c_S \max\{a_i^T x + b_i\mid i\in S\}\]
	with coefficients $c_S\in\Z$.
\end{theorem}

For the convex case, this yields a stronger version of \Cref{Thm:wang}, stating that any (not necessarily convex) CPWL function can be written as a linear combination of~\mbox{$(n+1)$-term} maxima. \Cref{Thm:main} is stronger in the sense that it guarantees that all pieces of the $(n+1)$-term maxima must be pieces of the original function. This makes it possible to bound the total number of these $(n+1)$-term maxima and, therefore, the size of an NN representing $f$, as we will see in the proof of the following theorem.

\begin{restatable}{theorem}{thmconvexbound}\label{thm:convex-bound}
	Let $f\colon\R^n\to\R$ be a convex CPWL function with $p$ affine pieces. Then~$f$ can be represented by a ReLU NN with depth $\lceil\log_2(n+1)\rceil+1$ and width $\mathcal{O}(p^{n+1})$.
\end{restatable}

\begin{proof}
	Using the representation of \Cref{Thm:main}, we can construct an NN computing $f$ by computing all the $(n+1)$-term max functions in parallel with the construction of \Cref{Lem:max} (similar to the proof by Arora et al.~\cite{Arora:DNNwithReLU} to show \Cref{Thm:CPWL}). This results in an NN with the claimed depth. Moreover, the width is at most a constant times the number of these $(n+1)$-term max functions. This number can be bounded in terms of the number of possible subsets $S\subseteq[p]$ with $\abs{S}\leq n+1$, which is at most~$p^{n+1}$.
\end{proof}

Before we present the proof of \Cref{Thm:main}, we show how we can generalize its consequences to the nonconvex case.

\subsection{The General (Nonconvex) Case}

It is a well-known fact that every CPWL function can be expressed as a difference of two convex CPWL functions, see, e.g., \cite[Theorem 1]{wang2004general}. This allows us to derive the general case from the convex case. What we need, however, is to bound the number of affine pieces of the two convex CPWL functions in terms of the number of pieces of the original function. Therefore, we consider a specific decomposition for which such bounds can easily be achieved.

\begin{proposition}\label{Prop:decomp}
	Let $f\colon\R^n\to\R$ be a CPWL function with $p$ affine pieces. Then, one can write~$f$ as $f=g-h$ where both $g$ and $h$ are convex CPWL functions with at most $p^{2n+1}$ pieces.
\end{proposition}
\begin{proof}
	Suppose the $p$ affine pieces of $f$ are given by $x\mapsto a_i^Tx+b_i$, $i\in[p]$. Define the function $h(x)\coloneqq\sum_{1\leq i<j \leq p}\max\{a_i^Tx+b_i,a_j^Tx+b_j\}$ and let $g\coloneqq f+h$. Then, obviously,~$f=g-h$. It remains to show that both $g$ and $h$ are convex CPWL functions with at most $p^{2n+1}$ pieces.
	
	The convexity of $h$ is clear by definition. Consider the ${p\choose2}=\frac{p(p-1)}{2}<p^2$ hyperplanes given by $a_i^Tx+b_i=a_j^Tx+b_j$, $1\leq i<j \leq p$. They divide $\R^n$ into at most~\mbox{${p^2\choose n}+{p^2\choose n-1}+\dots+{p^2\choose 0}\leq p^{2n}$} regions (compare \cite[Theorem~1.3]{edelsbrunner1987algorithms}) in each of which~$h$ is affine. In particular, $h$ has at most $p^{2n}\leq p^{2n+1}$ pieces.
	
	Next, we show that $g=f+h$ is convex. Intuitively, this holds because each possible breaking hyperplane of $f$ is made convex by adding $h$.
	To make this formal, note that by the definition of convexity, it suffices to show that $g$ is convex along each affine line. For this purpose, consider an arbitrary line $x(t)=ta+b$, $t\in\R$, given by $a\in\R^n$ and $b\in\R$. Let $\tilde{f}(t)\coloneqq f(x(t))$, $\tilde{g}(t)\coloneqq g(x(t))$, and $\tilde{h}(t)\coloneqq h(x(t))$. We need to show that $\tilde{g}\colon\R\to\R$ is a convex function. Observe that $\tilde{f}$, $\tilde{g}$, and $\tilde{h}$ are clearly one-dimensional CPWL functions with the property $\tilde{g}=\tilde{f}+\tilde{h}$. Hence, it suffices to show that $\tilde{g}$ is locally convex around each of its breakpoints. Let $t\in\R$ be an arbitrary breakpoint of $\tilde{g}$. If $\tilde{f}$ is already locally convex around $t$, then the same holds for $\tilde{g}$ as well since $\tilde{h}$ inherits convexity from $h$. Now suppose that $t$ is a nonconvex breakpoint of~$\tilde{f}$. Then there exist two distinct pieces of $f$, indexed by $i, j\in[p]$ with $i\neq j$, such that $\tilde{f}(t')=\min\{a_i^Tx(t')+b_i,a_j^Tx(t')+b_j\}$ for all $t'$ sufficiently close to $t$. By construction,~$\tilde{h}(t')$ contains the summand $\max\{a_i^Tx(t')+b_i,a_j^Tx(t')+b_j\}$. Thus, adding this summand to $\tilde{f}$ linearizes the nonconvex breakpoint of $\tilde{f}$, while adding all the other summands preserves convexity. In total, $\tilde{g}$ is locally convex around $t$, which finishes the proof that $g$ is a convex function.
	
	Finally, observe that pieces of $g=f+h$ are always intersections of pieces of $f$ and $h$, for which we have only $p\cdot p^{2n} = p^{2n+1}$ possibilities.
\end{proof}

Having this, we may conclude the following.

\thmwidthbound*

\begin{proof}
	Consider the decomposition $f=g-h$ from \Cref{Prop:decomp}. Using \Cref{thm:convex-bound}, we obtain that both $g$ and $h$ can be represented with the required depth $\lceil\log_2(n+1)\rceil+1$ and with width $\mathcal{O}((p^{2n+1})^{n+1})=\mathcal{O}(p^{2n^2+3n+1})$. Thus, the same holds true for $f$.
\end{proof}

\subsection{Extended Newton Polyhedra of Convex CPWL Functions}

For our proof of \Cref{Thm:main}, we use a correspondence of convex CPWL functions with certain polyhedra, which are known as (extended) Newton polyhedra in tropical geometry~\cite{maclagan2015introduction}. These relations between tropical geometry and neural networks have previously been applied to investigate expressivity of NNs; compare our references in \Cref{Sec:Lit}.

In order to formalize this correspondence, let $\CCPWL_n\subseteq\CPWL_n$ be the set of convex CPWL functions of type $\R^n\to\R$. For $f(x) = \max\{a_i^T x + b_i \mid i\in[p]\}$ in $\CCPWL_n$, we define its so-called \emph{extended Newton polyhedron} to be \[\mathcal{N}(f)\coloneqq\conv(\{(a_i^T, b_i)^T\in\R^n\times\R\mid i\in[p]\})+\cone(\{-e_{n+1}\})\subseteq\R^{n+1},\]
where the ``+'' stands for Minkowski addition.
We denote the set of all possible extended Newton polyhedra in $\R^{n+1}$ as $\Newt_n$. That is, $\Newt_n$ is the set of (unbounded) polyhedra in $\R^{n+1}$ that emerge from a polytope by adding the negative of the $(n+1)$-st unit vector $-e_{n+1}$ as an extreme ray. Hence, a set~\mbox{$P\subseteq\R^{n+1}$} is an element of~$\Newt_n$ if and only if~$P$ can be written as \[P=\conv(\{(a_i^T, b_i)^T\in\R^n\times\R\mid i\in[p]\})+\cone(\{-e_{n+1}\}).\] Conversely, for a polyhedron $P\in\Newt_n$ of this form, let~$\mathcal{F}(P)\in\CCPWL_n$ be the function defined by $\mathcal{F}(P)(x) = \max\{a_i^Tx + b_i \mid i\in[p]\}$.

There is an intuitive way of thinking about the extended Newton polyhedron $P$ of a convex CPWL function $f$: it consists of all hyperplane coefficients $(a^T,b)^T\in\R^n\times\R$ such that $a^Tx + b\leq f(x)$ for all $x\in\R^n$. This also explains why we add the extreme ray~$-e_{n+1}$: decreasing $b$ obviously maintains the property of $a^Tx + b$ being a lower bound on the function $f$. Hence, if a point $(a^T,b)^T$ belongs to the extended Newton polyhedron $P$, then also all points $(a^T,b')^T$ with $b'<b$ should belong to it. Thus, $-e_{n+1}$ should be contained in the recession cone of $P$.

In fact, there is a one-to-one correspondence between elements of $\CCPWL_n$ and $\Newt_n$, which is nicely compatible with some (functional and polyhedral) operations. This correspondence has been studied before in tropical geometry \cite{maclagan2015introduction, ETC}, convex geometry\footnote{$\mathcal{N}(f)$ is the negative of the epigraph of the convex conjugate of $f$.} \cite{HiriartUrrutyLemarechal93b}, as well as neural network literature \cite{Zhang:Tropical, charisopoulos2018tropical, alfarra2022decision, montufar2022sharp}.
We summarize the key findings about this correspondence relevant to our work in the following proposition:

\begin{proposition}\label{Prop:semiringiso}
	Let $n\in\N$ and $f_1,f_2\in\CCPWL_n$. Then it holds that
	\begin{enumerate}[(i)]
		\item the functions $\mathcal{N}\colon \CCPWL_n\to\Newt_n$ and $\mathcal{F}\colon \Newt_n\to\CCPWL_n$ are well-defined, that is, their output is independent from the representation of the input by pieces or vertices, respectively,
		\item $\mathcal{N}$ and $\mathcal{F}$ are bijections and inverse to each other,
		\item $\mathcal{N}(\max\{f_1,f_2\})=\conv(\mathcal{N}(f_1),\mathcal{N}(f_2))\coloneqq\conv(\mathcal{N}(f_1)\cup\mathcal{N}(f_2))$,
		\item $\mathcal{N}(f_1+f_2)=\mathcal{N}(f_1)+\mathcal{N}(f_2)$, where the $+$ on the right-hand side is Minkowski addition.
	\end{enumerate}
\end{proposition}
An algebraic way of phrasing this proposition is as follows: $\mathcal{N}$ and $\mathcal{F}$ are isomorphisms between the semirings $(\CCPWL_n,\max,+)$ and $(\Newt_n,\conv,+)$.

\subsection{Proof of \Cref{Thm:main}}

The rough idea to prove \Cref{Thm:main} is as follows.
Suppose we have a $p$-term max function $f$ with $p\geq n+2$. By \Cref{Prop:semiringiso}, $f$ corresponds to a polyhedron $P\in\Newt_n$ with at least $n+2$ vertices. Applying a classical result from discrete geometry known as \emph{Radon's theorem} allows us to carefully decompose $P$ into a ``signed''\footnote{Some polyhedra may occur with ``negative'' coefficents in that sum, meaning that they are actually added to $P$ instead of the other polyhedra. The corresponding CPWL functions will then have negative coefficients in the linear combination representing $f$.} Minkowski sum of polyhedra in $\Newt_n$ whose vertices are subsets of at most $p-1$ out of the $p$ vertices of~$P$.
Translating this back into the world of CPWL functions by \Cref{Prop:semiringiso} yields that $f$ can be written as linear combination of $p'$-term maxima with $p'<p$, where each of them involves a subset of the $p$ affine terms of $f$. We can then obtain \Cref{Thm:main} by iterating until every occurring maximum expression involves at most $n+1$ terms.

We start with a proposition that will be useful for our proof of \Cref{Thm:main}. Although its statement is well-known in the discrete geometry community, we include a proof for the sake of completeness.
To show the proposition, we make use of Radon's theorem (compare \cite[Theorem~4.1]{edelsbrunner1987algorithms}), stating that any set of at least $n+2$ points in $\R^n$ can be partitioned into two nonempty subsets such that their convex hulls intersect. 
\begin{proposition}\label{Prop:poly2.0}
	Given $p>n+1$ vectors $z_i=(a_i^T,b_i)^T\in \R^{n+1}$, $i\in[p]$, there exists a nonempty subset $U\subsetneq[p]$ featuring the following property: there is no~$c\in\R^{n+1}$ with~$c_{n+1}\geq0$ and~$\gamma\in\R$ such that
	\begin{align}
		\begin{split}
			c^Tz_i &>\gamma\quad\text{for all~$i\in U$, and}\\
			c^Tz_i &\leq\gamma\quad\text{for all~$i\in [p]\setminus U$.}	
		\end{split}
		\label{eq:prop}
	\end{align}
\end{proposition}

\begin{proof}
	Radon's theorem applied to the at least $n+2$ vectors $a_i$, $i\in[p]$, yields a nonempty subset \mbox{$U\subsetneq[p]$} and coefficients $\lambda_i\in[0,1]$ with $\sum_{i\in U}\lambda_i = \sum_{i\in[p]\setminus U}\lambda_i = 1$ such that $\sum_{i\in U}\lambda_i a_i = \sum_{i\in[p]\setminus U}\lambda_i a_i$. Suppose that~\mbox{$\sum_{i\in U}\lambda_ib_i \leq \sum_{i\in[p]\setminus U}\lambda_ib_i$} without loss of generality (otherwise exchange the roles of $U$ and $[p]\setminus U$).
	
	For any $c$ and $\gamma$ that satisfy \eqref{eq:prop} and $c_{n+1}\geq0$ it follows that
	\[
	\gamma <c^T \sum_{i\in U} \lambda_iz_i \leq c^T \sum_{i\in [p]\setminus U} \lambda_iz_i \leq \gamma,
	\]
	proving that no such $c$ and $\gamma$ can exist.
\end{proof}

The following proposition is a crucial step in order to show that any convex CPWL function with $p>n+1$ pieces can be expressed as an integer linear combination of convex CPWL functions with at most $p-1$ pieces.

\begin{proposition}\label{Prop:one_step}
	Let $f(x) = \max\{a_i^T x + b_i \mid i\in[p]\}$ be a convex CPWL function defined on $\R^n$ with $p>n+1$. Then there exist a subset $U\subseteq[p]$ such that
	\begin{equation}\label{Eq:evenodd}
		\sum_{\substack{W\subseteq U,\\\abs{W}\text{ even}}} \max\{a_i^T x + b_i\mid i\in[p]\setminus W\}
		=
		\sum_{\substack{W\subseteq U,\\\abs{W}\text{ odd}}} \max\{a_i^T x + b_i\mid i\in[p]\setminus W\}
	\end{equation}
\end{proposition}
\begin{proof}
	Consider the $p>n+1$ vectors $z_i\coloneqq(a_i^T,b_i)^T\in\R^{n+1}$, $i\in[p]$. Choose $U$ according to \Cref{Prop:poly2.0}. We show that this choice of $U$ guarantees equation~\eqref{Eq:evenodd}.
	
	For $W\subseteq U$, let $f_W(x)=\max\{a_i^T x + b_i\mid i\in[p]\setminus W\}$ and consider its extended Newton polyhedron $P_W=\mathcal{N}(f_W)=\conv(\{z_i\mid i\in[p]\setminus W\})+\cone(\{-e_{n+1}\})$. By \Cref{Prop:semiringiso}, equation~\eqref{Eq:evenodd} is equivalent to
	\[
	P_{\even}\coloneqq
	\sum_{\substack{W\subseteq U,\\\abs{W}\text{ even}}} P_W =\sum_{\substack{W\subseteq U,\\\abs{W}\text{ odd}}} P_W
	\eqqcolon P_{\odd},
	\]
	where the sums are Minkowski sums.
	
	We show this equation by showing that for all vectors $c\in\R^{n+1}$ it holds that
	\begin{equation}\label{Eq:LPequal}
		\max\{c^Tx\mid x\in P_{\even}\}=\max\{c^Tx\mid x\in P_{\odd}\}.
	\end{equation}
	
	Let $c\in\R^{n+1}$ be an arbitrary vector. If $c_{n+1}<0$, both sides of \eqref{Eq:LPequal} are infinite. Hence, from now on, assume that $c_{n+1}\geq 0$. Then, both sides of \eqref{Eq:LPequal} are finite since~$-e_{n+1}$ is the only extreme ray of all involved polyhedra.
	
	Due to our choice of $U$ according to \Cref{Prop:poly2.0}, there exists an index $u\in U$ such that
	\begin{equation}\label{Eq:notExtreme}
		c^Tz_u\leq\max_{i\in[p]\setminus U} c^Tz_i.
	\end{equation}
	We define a bijection $\varphi_u$ between the even and the odd subsets of $U$ as follows:
	\[
	\varphi_u(W) \coloneqq \left\{\begin{array}{ll}
		W\cup\{u\},&\text{if } u\notin W,\\
		W\setminus\{u\},&\text{if } u\in W.\\
	\end{array}\right.
	\]
	That is, $\varphi_u$ changes the parity of $W$ by adding or removing $u$.
	Considering the corresponding polyhedra $P_W$ and $P_{\varphi_u(W)}$, this means that $\varphi_u$ adds or removes the extreme point $z_u$ to or from $P_W$. Due to \eqref{Eq:notExtreme} this does not change the optimal value of maximizing in $c$-direction over the polyhedra, that is, \[\max\{c^Tx\mid x\in P_W\}=\max\{c^Tx\mid x\in P_{\varphi_u(W)}\}.\] Hence, we may conclude
	\begin{align*}
		\max\{c^Tx\mid x\in P_{\even}\}
		&=\sum_{\substack{W\subseteq U,\\\abs{W}\text{ even}}} \max\{c^Tx\mid x\in P_{W}\}\\
		&=\sum_{\substack{W\subseteq U,\\\abs{W}\text{ even}}} \max\{c^Tx\mid x\in P_{\varphi_u(W)}\}\\
		&=\sum_{\substack{W\subseteq U,\\\abs{W}\text{ odd}}} \max\{c^Tx\mid x\in P_{W}\}\\
		&=\max\{c^Tx\mid x\in P_{\odd}\},
	\end{align*}
	which proves \eqref{Eq:LPequal}. Thus, the claim follows.
\end{proof}

With the help of this result, we can now prove \Cref{Thm:main}.

\begin{proof}[Proof of \Cref{Thm:main}]
	Let $f(x) = \max\{a_i^T x + b_i \mid i\in[p]\}$ be a convex CPWL function defined on $\R^n$. Having a closer look at the statement of \Cref{Prop:one_step}, observe that only one term at the left-hand side of \eqref{Eq:evenodd} contains all $p$ affine combinations $a_i^Tx + b_i$. Putting all other maximum terms on the other side, we may write~$f$ as an integer linear combination of maxima of at most $p-1$ summands. Repeating this procedure until we have eliminated all maximum terms with more than $n+1$ summands yields the desired representation.
\end{proof}

\subsection{Potential Approaches to Show Lower Bounds on the Width}\label{sec:discuss_bounds}

In light of the upper width bounds shown in this section, a natural question to ask is whether also meaningful lower bounds can be achieved. This would mean constructing a family of CPWL functions with $p$ pieces defined on $\R^n$ (with different values of~$p$ and~$n$), for which we can prove that a large width is required to represent these functions with NNs of depth $\lceil\log_2(n+1)\rceil+1$.

A trivial and not very satisfying answer follows, e.g., from \cite{raghu2017expressive} or \cite{serra2018bounding}: for fixed input dimension $n$, they show that a function computed by an NN with $k$ hidden layers and width $w$ has at most~$\mathcal{O}(w^{kn})$ pieces. For our setting, this means that an NN with logarithmic depth needs a width of at least $\mathcal{O}(p^{1/(n\log n)})$ to represent a function with $p$ pieces. This is, of course, very far away from our upper bounds.

Similar upper bounds on the number of pieces have been proven by many other authors and are often used to show depth-width trade-offs \cite{montufar2014regions, montufar2022sharp, pascanu2014number, telgarsky2016benefits, Arora:DNNwithReLU}. However, there is a good reason why all these results only give rise to very trivial lower bounds for our setting: the focus is always on functions with considerably many pieces, which then, consequently, need many neurons to be represented (with small depth). However, since the lower bounds we strive for depend on the number of pieces, we would need to construct a family of functions with comparably few pieces that still need a lot of neurons to be represented.
In general, it seems to be a tough task to argue why such functions should exist.

A different approach could leverage methods from complexity theory, in particular from circuit complexity. Neural networks are basically arithmetic circuits with very special operations allowed. In fact, they can be seen as a tropical variant of arithmetic circuits. Showing circuit lower bounds is a notoriously difficult task in complexity theory, but maybe some conditional result (based on common conjectures similar to P~$\neq$~NP) could be established.

We think that the question whether our bounds are tight, or whether at least some non-trivial lower bounds on the width for NNs with logarithmic depth can be shown, is an exciting question for further research.

\section{Understanding Expressivity via Newton Polytopes}\label{Sec:Polytopes}

In \Cref{Sec:MIP}, we presented a mixed-integer programming approach towards proving that deep NNs can strictly represent more functions than shallow ones. However, even if we could prove that it is indeed enough to consider $H$-conforming NNs, this approach would not generalize to deeper networks due to computational limitations. Therefore, different ideas are needed to prove \Cref{Con:main} in its full generality. In this section, we point out that Newton polytopes of convex CPWL functions (similar to what we used in the previous section) could also be a way of proving \Cref{Con:main}. Using a homogenized version of \Cref{Prop:semiringiso}, we provide an equivalent formulation of \Cref{Con:main} that is completely phrased in the language of discrete geometry.

Recall that, by \Cref{prop:wlognobias}, we may restrict ourselves to NNs without biases. In particular, all CPWL functions represented by such NNs, or parts of it, are positively homogeneous. For the associated extended Newton polyhedra (compare \Cref{Prop:semiringiso}), this has the following consequence: all vertices $(a,b)\in\R^n\times\R$ lie in the hyperplane~$b=0$, that is, their $(n+1)$-st coordinate is $0$. Therefore, the extended Newton polyhedron of a positively homogeneous, convex CPWL function $f(x)=\max\{a_i^T x\mid i\in[p]\}$ is completely characterized by the so-called \emph{Newton polytope}, that is, the polytope~\mbox{$\conv(\{a_i\mid i\in[p]\})\subseteq\R^n$}.

To make this formal, let $\overbar{\CCPWL}_n$ be the set of all positively homogeneous, convex CPWL functions of type $\R^n\to\R$ and let $\overbar{\Newt}_n$ be the set of all convex polytopes in $\R^n$. Moreover, for $f(x)=\max\{a_i^T x\mid i\in[p]\}$ in $\overbar{\CCPWL}_n$, let \[\overbar{\mathcal{N}}(f)\coloneqq\conv(\{a_i\mid i\in[p]\})\in\overbar{\Newt}_n\] be the associated Newton polytope of $f$ and for $P=\conv(\{a_i\mid i\in[p]\})\in\overbar{\Newt}_n$ let \[\overbar{\mathcal{F}}(P)(x)=\max\{a_i^T x\mid i\in[p]\}\] be the so-called associated \emph{support function}~\cite{lemarechal1996convex} of $P$ in $\overbar{\CCPWL}_n$.
With this notation, we obtain the following variant of \Cref{Prop:semiringiso}.

\begin{proposition}\label{Prop:semiringisoII}
	Let $n\in\N$ and $f_1,f_2\in\overbar{\CCPWL}_n$. Then it holds that
	\begin{enumerate}[(i)]
		\item the functions $\overbar{\mathcal{N}}\colon \overbar{\CCPWL}_n\to\overbar{\Newt}_n$ and $\overbar{\mathcal{F}}\colon \overbar{\Newt}_n\to\overbar{\CCPWL}_n$ are well-defined, that is, their output is independent from the representation of the input by pieces or vertices, respectively,
		\item $\overbar{\mathcal{N}}$ and $\overbar{\mathcal{F}}$ are bijections and inverse to each other,
		\item $\overbar{\mathcal{N}}(\max\{f_1,f_2\})=\conv(\overbar{\mathcal{N}}(f_1),\overbar{\mathcal{N}}(f_2))\coloneqq\conv(\overbar{\mathcal{N}}(f_1)\cup\overbar{\mathcal{N}}(f_2))$,
		\item $\overbar{\mathcal{N}}(f_1+f_2)=\overbar{\mathcal{N}}(f_1)+\overbar{\mathcal{N}}(f_2)$, where the $+$ on the right-hand side is Minkowski addition.
	\end{enumerate}
\end{proposition}
In other words, $\overbar{\mathcal{N}}$ and $\overbar{\mathcal{F}}$ are isomorphisms between the semirings $(\overbar{\CCPWL}_n,\max,+)$ and $(\overbar{\Newt}_n,\conv,+)$.

Next, we study which polytopes can appear as Newton polytopes of convex CPWL functions computed by NNs with a certain depth; compare Zhang et al.~\cite{Zhang:Tropical}.

Before we apply the first ReLU activation, any function computed by an NN is linear. Thus, the corresponding Newton polytope is a single point. Starting from that, let us investigate a neuron in the first hidden layer. Here, the ReLU activation function computes a maximum of a linear function and~$0$. Therefore, the Newton polytope of the resulting function is the convex hull of two points, that is, a line segment. After the first hidden layer, arbitrary many functions of this type can be added up. For the corresponding Newton polytopes, this means that we take the Minkowski sum of line segments, resulting in a so-called \emph{zonotope}.

Now, this construction can be repeated layerwise, making use of \Cref{Prop:semiringisoII}: in each hidden layer, we can compute the maximum of two functions computed by the previous layers, which translates to obtaining the new Newton polytope as a convex hull of the union of the two original Newton polytopes. In addition, the linear combinations between layers translate to scaling and taking Minkowski sums of Newton polytopes.

This intuition motivates the following definition. Let $\overbar{\Newt}_n^{(0)}$ be the set of all polytopes in $\R^n$ that consist only of a single point. Then, for each $k\geq 1$, we recursively define
\begin{align*}
	\overbar{\Newt}_n^{(k)}\coloneqq\left\{\sum_{i=1}^{p}\conv(P_i,Q_i)\st P_i,Q_i\in\overbar{\Newt}_n^{(k-1)},\ p\in\N\right\},
\end{align*}
where the sum is a Minkowski sum of polytopes. A first, but not precisely accurate interpretation is as follows: the set $\overbar{\Newt}_n^{(k)}$ contains the Newton polytopes of positively homogeneous, convex CPWL functions representable with a $k$-hidden-layer NN. See \Cref{Fig:Newton} for an illustration of the case $k=2$.

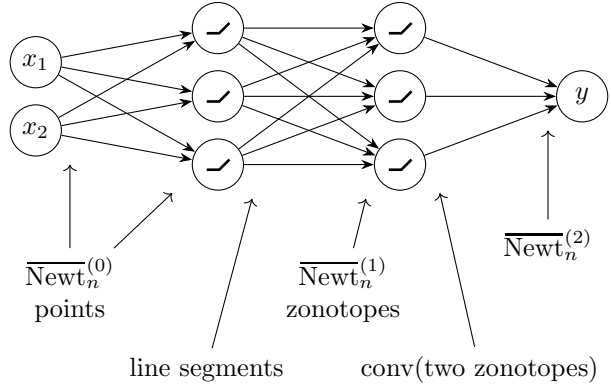
\begin{figure}[t]
	\centering
	\begin{tikzpicture}[every node/.style={transform shape}]\small
		\node[smallneuron] (x1) at (0,6ex) {$x_1$};
		\node[smallneuron] (x2) at (0,0) {$x_2$};
		\node[smallneuron] (n11) at (16ex,9ex) {\relu};
		\node[smallneuron] (n12) at (16ex,3ex) {\relu};
		\node[smallneuron] (n13) at (16ex,-3ex) {\relu};
		\draw[connection] (x1) -- (n11);
		\draw[connection] (x1) -- (n12);
		\draw[connection] (x1) -- (n13);
		\draw[connection] (x2) -- (n11);
		\draw[connection] (x2) -- (n12);
		\draw[connection] (x2) -- (n13);
		\node[smallneuron] (n21) at (32ex,9ex) {\relu};
		\node[smallneuron] (n22) at (32ex,3ex) {\relu};
		\node[smallneuron] (n23) at (32ex,-3ex) {\relu};
		\draw[connection] (n11) -- (n21);
		\draw[connection] (n12) -- (n21);
		\draw[connection] (n13) -- (n21);
		\draw[connection] (n11) -- (n22);
		\draw[connection] (n12) -- (n22);
		\draw[connection] (n13) -- (n22);
		\draw[connection] (n11) -- (n23);
		\draw[connection] (n12) -- (n23);
		\draw[connection] (n13) -- (n23);
		\node[smallneuron] (y) at (48ex,3ex) {$y$};
		\draw[connection] (n21) -- (y);
		\draw[connection] (n22) -- (y);
		\draw[connection] (n23) -- (y);
		
		\node[align=center] (a1) at (3ex,-14ex) {$\overbar{\Newt}_n^{(0)}$\\points};
		\draw[->] (a1) -- (3ex, -3ex);
		\draw[->] (a1) -- (12ex, -5.5ex);
		
		\node (o1) at (15ex,-21ex) {line segments};
		\draw[->] (o1) -- (19ex, -6ex);
		
		\node[align=center] (a2) at (27ex,-14ex) {$\overbar{\Newt}_n^{(1)}$\\zonotopes};
		\draw[->] (a2) -- (29ex, -6ex);
		
		\node (o2) at (39ex,-21ex) {$\conv(\text{two zonotopes})$};
		\draw[->] (o2) -- (35.5ex, -5.5ex);
		
		\node (a3) at (45ex,-10ex) {$\overbar{\Newt}_n^{(2)}$};
		\draw[->] (a3) -- (45ex, 0ex);
	\end{tikzpicture}
	\caption{Set of polytopes that can arise as Newton polytopes of convex CPWL functions computed by (parts of) a 2-hidden-layer NN.}
	\label{Fig:Newton}
\end{figure}

Unfortunately, this interpretation is not accurate for the following reason: our NNs are allowed to have negative weights, which cannot be fully captured by Minkowski sums as introduced above. Therefore, it might be possible that a $k$-hidden-layer NN can compute a convex function with Newton polytope not in $\overbar{\Newt}_n^{(k)}$. Luckily, one can remedy this shortcoming, and even extend the interpretation to the non-convex case, by representing the computed function as difference of two convex functions.

\begin{theorem}\label{Thm:Newton}
	A positively homogeneous (not necessarily convex) CPWL function can be computed by a $k$-hidden-layer NN if and only if it can be written as the difference of two positively homogeneous, convex CPWL functions with Newton polytopes in~$\overbar{\Newt}_n^{(k)}$.
\end{theorem}
\begin{proof}
	We use induction on $k$. For $k=0$, the statement is clear since it holds precisely for linear functions. For the induction step, suppose that, for some $k\geq1$, the equivalence is valid up to $k-1$ hidden layers. We prove that it is also valid for $k$ hidden layers.
	
	We need to show two directions. For the first direction, assume that $f$ is an arbitrary, positively homogeneous CPWL function that can be written as $f=g-h$ with \mbox{$\overbar{\mathcal{N}}(g),\overbar{\mathcal{N}}(h)\in\overbar{\Newt}_n^{(k)}$}. We need to show that a $k$-hidden-layer NN can compute~$f$. We show that this is even true for~$g$ and $h$, and hence, also for $f$. By definition of $\overbar{\Newt}_n^{(k)}$, there exist a finite number~$p\in\N$ and polytopes $P_i, Q_i\in \overbar{\Newt}_n^{(k-1)}$, $i\in[p]$, such that $\overbar{\mathcal{N}}(g)=\sum_{i=1}^p\conv(P_i,Q_i)$. By \Cref{Prop:semiringisoII}, we have
	\mbox{$
		g=\sum_{i=1}^p \max\{\overbar{\mathcal{F}}(P_i),\overbar{\mathcal{F}}(Q_i)\}
		$}.
	By induction, $\overbar{\mathcal{F}}(P_i)$ and~$\overbar{\mathcal{F}}(Q_i)$ can be computed by NNs with $k-1$ hidden layers. Since the maximum terms can be computed with a single hidden layer, in total a~$k$-th hidden layer is sufficient to compute $g$. An analogous argument applies to $h$. Thus, $f$ is computable with $k$ hidden layers, completing the first direction.
	
	For the other direction, suppose that $f$ is an arbitrary, positively homogeneous CPWL function that can be computed by a $k$-hidden-layer NN.
	Let us separately consider the~$n_k$ neurons in the $k$-th hidden layer of the NN. Let $a_i$, $i\in[n_k]$, be the weight of the connection from the $i$-th neuron in that layer to the output. Without loss of generality, we have $a_i\in\{\pm1\}$, because otherwise we can normalize it and multiply the weights of the incoming connections to the $i$-th neuron with $\abs{a_i}$ instead. Moreover, let us assume that, by potential reordering, there is some $m\leq n_k$ such that $a_i=1$ for $i\leq m$ and~$a_i=-1$ for $i>m$. With these assumptions, we can write
	\begin{equation}\label{eq:outputlayer}
		f=\sum_{i=1}^{m} \max\{0,f_i\} - \sum_{i=m+1}^{n_k} \max\{0,f_i\},
	\end{equation}
	where each $f_i$ is computable by a $(k-1)$-hidden-layer NN, namely the sub-NN computing the input to the $i$-th neuron in the $k$-th hidden layer.
	
	By induction, we obtain $f_i=g_i-h_i$ for some positively homogeneous, convex functions~$g_i,h_i$ with $\overbar{\mathcal{N}}(g_i),\overbar{\mathcal{N}}(h_i)\in \overbar{\Newt}_n^{(k-1)}$. We then have
	\begin{equation}\label{eq:perneuron}
		\max\{0,f_i\}=\max\{g_i,h_i\}-h_i.
	\end{equation}
	We define
	\[
	g\coloneqq\sum_{i=1}^m \max\{g_i,h_i\} + \sum_{i=m+1}^{n_k} h_i
	\]
	and 
	\[
	h\coloneqq\sum_{i=1}^m h_i + \sum_{i=m+1}^{n_k} \max\{g_i,h_i\}.
	\]
	
	Note that $g$ and $h$ are convex by construction as a sum of convex functions and that \eqref{eq:outputlayer} and \eqref{eq:perneuron} imply $f=g-h$. 
	Moreover, by \Cref{Prop:semiringisoII}, \[\overbar{\mathcal{N}}(g)=\sum_{i=1}^m \conv(\overbar{\mathcal{N}}(g_i),\overbar{\mathcal{N}}(h_i)) + \sum_{i=m+1}^{n_k} \conv(\overbar{\mathcal{N}}(h_i),\overbar{\mathcal{N}}(h_i))\in\overbar{\Newt}_n^{(k)}\] and \[\overbar{\mathcal{N}}(h)=\sum_{i=1}^m \conv(\overbar{\mathcal{N}}(h_i),\overbar{\mathcal{N}}(h_i)) + \sum_{i=m+1}^{n_k} \conv(\overbar{\mathcal{N}}(g_i),\overbar{\mathcal{N}}(h_i))\in\overbar{\Newt}_n^{(k)}.\] Hence, $f$ can be represented as desired, completing also the other direction.
\end{proof}

The power of \Cref{Thm:Newton} lies in the fact that it provides a purely geometric characterization of the class $\ReLU(k)$. The classes of polytopes $\overbar{\Newt}_n^{(k)}$ are solely defined by the two simple geometric operations Minkowski sum and convex hull of the union. Therefore, understanding the class $\ReLU(k)$ is equivalent to understanding what polytopes one can generate by iterative application of these geometric operations.

In particular, we can give yet another equivalent reformulation of our main conjecture. To this end, let the simplex $\Delta_n\coloneqq\conv\{0,e_1,\dots,e_n\}\subseteq\R^n$ denote the Newton polytope of the function $f_n=\max\{0,x_1,\dots,x_n\}$ for each $n\in \N$.

\begin{conjecture}\label{Con:polytopes}
	For every $k\in\N$, $n=2^k$, there does not exist a pair of polytopes \mbox{$P,Q\in\overbar{\Newt}_n^{(k)}$} with $\Delta_n+Q=P$ (Minkowski sum).
\end{conjecture}

\begin{theorem}
	\Cref{Con:polytopes} is equivalent to \Cref{Con:main} and \Cref{Con:simplified}.
\end{theorem}
\begin{proof}
	By \Cref{prop:equivalence}, it suffices to show equivalence between \Cref{Con:polytopes} and \Cref{Con:simplified}. By \Cref{Thm:Newton}, $f_n$ can be represented with $k$ hidden layers if and only if there are functions $g$ and $h$ with Newton polytopes in $\overbar{\Newt}_n^{(k)}$ satisfying $f_n+h=g$. By \Cref{Prop:semiringisoII}, this happens if and only if there are polytopes $P,Q\in \overbar{\Newt}_n^{(k)}$ with~$\Delta_n+Q=P$.
\end{proof}

It is particularly interesting to look at special cases with small $k$. For $k=1$, the set~$\overbar{\Newt}_n^{(1)}$ is the set of all zonotopes. Hence, the (known) statement that $\max\{0,x_1,x_2\}$ cannot be computed with one hidden layer~\cite{mukherjee2017lower} is equivalent to the fact that the Minkowski sum of a zonotope and a triangle can never be a zonotope.

The first open case is the case $k=2$. An unconditional proof that two hidden layers do not suffice to compute the maximum of five numbers is highly desired. In the regime of Newton polytopes, this means to understand the class $\overbar{\Newt}_n^{(2)}$. It consists of finite Minkowski sums of polytopes that arise as the convex hull of the union of two zonotopes. Hence, the major open question here is to classify this set of polytopes.

Finally, let us remark that there exists a generalization of the concept of polytopes, known as \emph{virtual polytopes}~\cite{virtual}, that makes it possible to assign a Newton polytope also to non-convex CPWL functions. This makes use of the fact that every (non-convex) CPWL function is a difference of two convex ones. Consequently, a virtual polytope is a formal Minkowski difference of two ordinary polytopes. Using this concept, \Cref{Thm:Newton} and \Cref{Con:polytopes} can be phrased in a simpler way, replacing the pair of polytopes with a single virtual polytope.

\section{Future Research}\label{Sec:discuss}

The most obvious and, at the same time, most exciting open research question is to prove or disprove \Cref{Con:main}, or equivalently \Cref{Con:simplified} or \Cref{Con:polytopes}. The first step could be to prove that it is indeed enough to consider $H$-conforming NNs. This is intuitive because every breakpoint introduced at any place outside the hyperplanes $H_{ij}$ needs to be canceled out later. Therefore, it is natural to assume that these breakpoints do not have to be introduced in the first place. However, this intuition does not seem to be enough for a formal proof because it could occur that additional breakpoints in intermediate steps, which are canceled out later, also influence the behavior of the function at other places where we allow breakpoints in the end.

Another step towards resolving our conjecture may be to find an alternative proof of \Cref{Thm:maxWithAssumption}, not using $H$-conforming NNs. This might also be beneficial for generalizing our techniques to more hidden layers, since, while theoretically possible, a direct generalization of the MIP approach is infeasible due to computational limitations. For example, it might be particularly promising to use a tropical approach as described in \Cref{Sec:Polytopes} and apply methods from polytope theory to prove \Cref{Con:polytopes}.

In light of our results from \Cref{sec:more}, it would be desirable to provide a complete characterization of the functions contained in $\ReLU(k)$.
Another potential research goal is improving our upper bounds on the width from \Cref{sec:width} and/or proving matching lower bounds as discussed in \Cref{sec:discuss_bounds}.

Some more interesting research directions are the following:
\begin{itemize}
	\item establishing or strengthening our results for special classes of NNs like recurrent neural networks (RNNs) or convolutional neural networks (CNNs),
	\item using exact representation results to show more drastic depth-width trade-offs compared to existing results in the literature,
	\item understanding how the class $\ReLU(k)$ changes when a polynomial upper bound is imposed on the width of the NN; see related work by Vardi et al.~\cite{vardi2021size}.
	\item understanding which CPWL functions one can (exactly) represent with polynomial size at all, without any restriction on the depth; see related work in the context of combinatorial optimization~\cite{maxflowPaper,knapsackPaper}.
\end{itemize}

\bibliographystyle{abbrv}
\bibliography{cpwl, full-bib}

\end{document}